%% file: main.tex
\documentclass[twoside,11pt]{article}
\usepackage{blindtext}

\usepackage{amsmath}
\let\proof\relax 

\usepackage{mathrsfs,amsthm}
\usepackage{algorithmic}
\usepackage{algorithm}
\usepackage{array}
\usepackage[caption=false,font=normalsize,labelfont=sf,textfont=sf]{subfig}
\usepackage{textcomp}
\usepackage{url}
\usepackage{verbatim}
\usepackage{graphicx}
\usepackage{amsfonts,amssymb}
\usepackage{xcolor,bm,setspace}
\usepackage{mathtools}
\usepackage{enumitem}
\usepackage{jmlr2e}
\newcommand\pl{\stackrel{PL(p)}{=}}
\newcommand\plconv{\stackrel{PL(2)}{\longrightarrow}}
\def\Exp{\mathbb{E}\,}

\renewcommand{\tilde}{\widetilde}

\providecommand{\wh}{\widehat}
\newcommand{\bs}[1]{\mathbf{#1}}
\renewcommand{\bar}{\overline}
\providecommand{\Real}{\mathbb{R}}
\DeclareMathOperator*{\argmin}{argmin}
\newcommand{\tl}[1]{\widetilde{{#1}}}
\newcommand{\mc}[1]{\mathcal{#1}}
\def\tr{{\rm tr}}

\usepackage{lastpage}
\firstpageno{1}

\begin{document}

\title{Estimation of Embedding Vectors in High Dimensions}

\author{\name Golara Ahmadi Azar \email golazar@g.ucla.edu \\
       \name Melika Emami\thanks{Now with Optum AI Labs, work done while at UCLA}\email emami@g.ucla.edu \\
       \name Alyson Fletcher \email akfletcher@ucla.edu \\
       \addr University of California Los Angeles (UCLA),\\
       Los Angeles, CA USA
       \AND
       \name Sundeep Rangan \email srangan@nyu.edu \\
       \addr New York University (NYU),\\
       Brooklyn, NY USA}

% For research notes, remove the comment character in the line below.
% \researchnote
\editor{}
\maketitle

\begin{abstract}
Embeddings are a basic initial feature extraction step
in many machine learning models, particularly
in natural language processing.  
An embedding attempts to map data tokens
to a low-dimensional space where
 similar tokens are mapped to vectors
that are close to one another by some
metric in the embedding space.
A basic question is how well can 
such embedding be learned?
To study this problem,
we consider a simple probabilistic model
for discrete data where there is some ``true" but unknown embedding where the  correlation
of random variables is related
to the similarity of the embeddings.
Under this model, it is shown that
the embeddings can be learned by a
variant of low-rank approximate
message passing (AMP) method.
The AMP approach enables precise
predictions of the accuracy of the estimation
in certain high-dimensional limits.
In particular, the methodology
provides insight on the 
relations of key parameters such
as the number of samples per 
value, the frequency of the terms,
and the strength of the embedding correlation
on the probability distribution.
 Our theoretical findings are validated by simulations on both synthetic data and real text data.
\end{abstract}


%----------------SECTIONS--------------------------
\input{sections/introduction}
\input{sections/problem_formulation}
\input{sections/algorithm}

\input{sections/results}
\input{sections/lsl_proofs}

\input{sections/simulation}
\input{sections/conclusions}

%\bibliography{ref}
\end{document}

%% file: sections/introduction.tex
\section{Introduction}
\label{sec:introduction}
Embeddings are widely-used in machine learning tasks, particularly text processing \cite{Asudani2023}.
In this work, we study embedding learned on  pairs of discrete random variables, $(X_1,X_2)$, 
where
$X_1  \in [m] := \{1,\ldots,m\}$ and $X_2 \in [n] := \{1,\ldots,n\}$.
For example, in word embeddings, $X_1$ could represent a target word, and $X_2$ a context word (e.g., a second word found close to the target word) \cite{pennington2014glove}.
By an \emph{embedding}, we mean a pair of
mappings of the form:
\begin{equation}
    X_1=i \mapsto \bs{u}_i, 
    \quad
    X_2=j \mapsto \bs{v}_j,
\end{equation}
where $\bs{u}_i$ and $\bs{v}_j \in \Real^d$.
The embedding thus maps 
each value of the random variable
to an associated $d$-dimensional 
vector.  The dimension $d$ is called
the \emph{embedding dimension}.

Typically, (see e.g., \cite{pennington2014glove}), we try to learn embeddings such that $\bs{u}_i^\intercal \bs{v}_j$ is large
when the pair $(X_1,X_2)=(i,j)$
occurs more frequently. 
Many algorithms have been proposed for training such embeddings \cite{mikolov2013efficient,Stein_2019,pennington2014glove,joulin-etal-2017-bag}. While these algorithms have been successful in practice, precise convergence results are difficult to obtain. 
At root, we wish to understand how
well can embeddings be learned?
% For example, questions include:
% how well do the correlations, $\bs{u}_i^\intercal\bs{v}_j$ of learned embeddings
% predict the underlying correlation of events $X_i=i$ and $X_2=j$.  How do these predictions
% depend on the number of data samples
% available and embedding dimension?

To study these problems, 
we propose a simple model for the joint distribution of $(X_1,X_2)$ where
\begin{equation} \label{eq:pcorr}
    \log\left[ \frac{P(X_1=i,X_2=j)}{P(X_1=i)P(X_2=j)} \right] \approx \frac{1}{\sqrt{m}}\bs{u}_i^\intercal \bs{v}_j,
\end{equation}
for some true embedding vectors $\bs{u}_i$ and $\bs{v}_j$.
The property \eqref{eq:pcorr}
indicates that
the pointwise mutual information (PMF) %log correlation 
 of the events that
$X_1=i$ and $X_2=j$ is proportional
to the vector correlation 
$\bs{u}_i^\intercal \bs{v}_j$
in the embedding space so that a large $\bs{u}_i^\intercal \bs{v}_j$
implies that $(X_1,X_2)=(i,j)$ occurs relatively frequently.
The model also has parameters $s_i^u$ and $s_j^v$ such that
the marginal distributions (which we call the \emph{bias} terms) are given by
%\begin{subequations}
\begin{align}
    P(X_1=i) \propto  \exp(s_i^u), \quad 
     P(X_2=j) \propto  \exp(s_j^v).
\end{align}
%\end{subequations}
The problem is to estimate the true bias terms 
and the embedding vectors from samples $(x_1,x_2)=(i,j)$.
We consider Maximum Likelihood (ML) estimation of the parameters.  In our probabilistic model,
the ML estimation can be approximated by a low-rank matrix factorization \cite{Kumar2017},
which are widely-used in learning embeddings \cite{pennington2014glove,Lee2000NMF}.

The low-rank matrix factorization is analyzed in a certain large system limit (LSL).
Specifically, the embedding dimension $d$ is fixed while the number of terms $n$
and $m$ (equivalent to the vocabulary size in word embeddings) 
and the average number of samples grow to infinity in a certain scaling.
The true bias and embedding parameters are generated randomly,
and we examine how well an approximation of ML estimation is able to
recover the parameters.  In practice, most embeddings
are learned via stochastic gradient descent or related algorithms.
In this work, we analyze a variant of low-rank approximate message passing (AMP)
methods.  Several AMP methods are available for
low-rank matrix factorization
(AMP-KM \cite{matsushita2013cluster}, IterFac\cite{DBLP:journals/corr/abs-1202-2759}, Low-rank AMP \cite{lesieur2015mmse}).  
The main benefit of the AMP
is that the framework enables precise predictions of the performance in the large
system limit.  

Our contributions are as follows:
\begin{itemize}
    \item \emph{Extension of low-rank AMP:}  Our method is most closely related
    to the low-rank AMP algorithms of \cite{lesieur2015mmse} that considers
    estimates of low-rank matrices under general non-Gaussian measurements.
    We show that this method, however, cannot directly be applied to the 
    problem of learning embeddings due to the presence of the bias terms
    $s_i^u$ and $s_j^v$.  We develop an extension for the low-rank AMP 
    that we call biased low-rank AMP that can account for the variations 
    due to the bias terms.  

    \item \emph{State evolution analysis:}  Similar to other AMP algorithms 
    \cite{Donoho_2009,bayati2011dynamics}, we provide a precise characterization
    of the joint distribution of the true vectors, the bias terms and
    their estimates.  The distribution is described in each iteration of the AMP
    algorithm through a \emph{state evolution} or SE.  
    From the joint distribution,
    one can evaluate various performance
    metrics such as mean squared error (MSE) or overlap of 
    the true and learned embedding vectors
    as well as the error in the learned
    joint probability distribution.
    The performance, in turn, can be related
    to key parameters such as the number
    of data samples per outcome $(i,j)$,
    the relative frequency of terms,
    and strength of the dependence of the
    embedding correlation $\bs{u}_i^{\intercal} \bs{v}_j$ on the correlation of events $X_1=i$ and $X_2=j$.

    \item \emph{Experimental results:}  The predictions from the
    SE analysis are validated on both synthetic datasets as well 
    as a text dataset from movie reviews \cite{maas2011sent}. 
    While the ``ground" truth embeddings vectors in the movie 
    dataset are not known, we propose a novel evaluation method,
    where we learn ``true" vectors from a large number of samples
    and then predict the performance on smaller numbers.
\end{itemize}

A shorter version of this paper was presented at the $58$th Annual Conference on Information Sciences and Systems (CISS) \cite{azar2024ciss}. In the current paper, we have provided significantly more details on the proofs, and added more experimental results supporting our hypothesis.

\textbf{Prior work:}  Learned  embeddings are widely-used in  applications in Natural Language Processing \cite{Asudani2023-lx}, Computer Vision \cite{Wu_2017_ICCV} (e.g. zero-shot learning \cite{bucher2016improving,Zhang_2017_CVPR}, contrastive learning \cite{Han_2021_CVPR}, and face recognition \cite{chopra2005face,schroff2015facenet}), graph and network representation learning \cite{Hiraoka2024node2vec,Fatemi2023hyper,Davison2023network}, surrogate loss function design \cite{Finocchiaro2024loss} and even biosignal based inference \cite{azar2024emb}. Despite the empirical success of the numerous embedding learning techniques (see \cite{Asudani2023-lx} and references therein), there is limited theoretical analysis of the asymptotic behavior of the learned embeddings \cite{grohe2020}, especially in high dimensional limits.

However, it is well-known that most embedding methods are closely-related
to finding low-rank matrix approximations  \cite{pennington2014glove,Lee2000NMF}.
AMP algorithms provide a tractable approach to rigorously analyzing low-rank 
estimation problems in high-dimensional limits.
AMP algorithms were originally developed for compressed sensing
problems \cite{Donoho_2009,ziniel2013csamp}. For example, authors of \cite{huang2022amppe} explore one/multi-bit compressive sensing problems via AMP where the signal and noise distribution parameters are treated as variables and jointly recovered. In \cite{ma2019amp}, authors present the AMP-SI algorithm that utilizes side information (SI) to aid in signal recovery using conditional denoisers.
These algorithms have also been widely-used in analysis
of low-rank estimation problems.
Early AMP-based low-rank estimation algorithms  
were introduced by  \cite{matsushita2013cluster}
and \cite{fletcher2018inference}.

AMP methods were proven to be optimal for the case of sparse PCA \cite{deshpande2014informationtheoretically}.
The work \cite{deshpande2015asymptotic} applied AMP to the stochastic block model which is a popular statistical model for the large-scale structure of complex networks.
Authors \cite{montanari2016npca} address the shortcomings of classical PCA in the high dimensional and low SNR regime. They use an AMP algorithm to solve the non-convex non-negative PCA problem. In \cite{Kabashima_2016}, the 
authors consider a general form of the problem at hand and provide the MMSE that is in principle achievable in any computational time. Specifically relevant to our study, \cite{lesieur2015mmse,Lesieur_2017} present a framework to address the constrained low-rank matrix estimation assuming a general prior on the factors, and a general output channel (a biased Poisson channel in our case) through which the matrix is observed. Noting that state evolution is uninformative when the algorithm is initialized near an unstable fixed point, \cite{montanari2019estimation} proposes a new analysis of AMP that allows for spectral initializations.  The main contribution
of the current work is to modify and apply these methods to the embedding learning problem.Finally, we would like to emphasize that our proposed method is to provide a framework that helps us understand the relations between key parameters in an estimation model featuring static embeddings and unknown biases, rather than providing an alternative to state of the art NLP algorithms \cite{devlin2019bert,radford2018improving}. 

%% file: sections/problem_formulation.tex
\section{Problem Formulation}
\label{sec:problem-formulation}
\subsection{Joint Density Model for the Embedding}
\label{subsec:poisson}
As stated in the introduction,
we consider embeddings of pairs of
discrete random variables $(X_1,X_2)$
with $X_1 \in [m]$ and $X_2 \in [n]$
for some $m$ and $n$.
Let $P_i^{(1)} = P(X_1=i)$ and
$P_j^{(2)}=P(X_2=j)$ denote the marginal
distributions 
and $P_{ij}=P(X_i=i,X_2=j)$
denote the joint distribution.
We assume the joint distribution has the form,
\begin{align} \label{eq:Pij}
    P_{ij} = C \exp\left( \frac{1}{\sqrt{m}}\bs{u}_i^\intercal \bs{v}_j + s_i^u + s_j^v \right), %\label{eq:pmf}
\end{align}
where $\bs{u}_i, \bs{v}_j \in \Real^d$
are some ``true" embedding vectors,
$s_i^u$ and $s_j^v$ are scalars, 
and $C > 0$ is a normalization 
constant.  It can be verified that,
for large $m$, the marginal distributions
of $X_1$ are $X_2$ satisfy:
\begin{subequations}\label{eq:marginal}
    \begin{align}
    \log P_i^{(1)} = C_1 +  s_i^u + O(1/\sqrt{m}), \\
    \log P_j^{(2)} = C_2 +  s_j^v + O(1/\sqrt{m}),
\end{align}
\end{subequations}
where $C_1$ and $C_2$ are constants.
Hence, $s_i^u$ and $s_j^v$, which 
we will call the \emph{bias} terms,
represent the log likelihoods of the values.
Also, the PMF \eqref{eq:Pij}
satisfies the property 
\begin{align}
    \log \left[ \frac{P_{ij}}{P_i^{(1)}P_j^{(2)}} \right] = \frac{1}{\sqrt{m}}
    \bs{u}_i^\intercal \bs{v}_j + O(1/m),
\end{align}
Hence, the similarity $\bs{u}_i^\intercal \bs{v}_j$ represents the log of the correlation of the events that 
$X_1=i$ and $X_2=j$.

% For the rest of the paper and without loss of generality, we assume $\lambda_0=1$.
% Our goal is to estimate the matrix $M$ and bias terms $s^u$ and $s^v$.

\subsection{Poisson Measurements}\label{sec:poisson}
The parameters to estimate in the model \eqref{eq:Pij} are:
\begin{equation} \label{eq:theta}
    \theta := (U, V, \bs{s}^u, \bs{s}^v),
\end{equation}
where $U$ and $V$ are the matrices
with embedding vectors $\bs{u}_i$ and $\bs{v}_j$, and $\bs{s}^u$ and $\bs{s}^v$ are the vectors
of the bias terms $s_i^u$ and $s_j^v$. 
To learn the parameters, we are given a set of samples, $(x_1^t,x_2^t)$, $t=1,\ldots, N$.  
Let
\begin{equation} 
    Z_{ij} = \left|\{ t ~|~ (x_1^t=i,x_2^t=j) \} \right|,
\end{equation}
which are the number of instances
where $(X_1,X_2)=(i,j)$.
If we assume that the samples
are independent and identicaly distributed (i.i.d.), with PMF \eqref{eq:Pij}
and the number of samples, $N$,
is Poisson distributed, then the measurements $Z_{ij}$
will be independent with distributions,
\begin{align}
    Z_{ij} \sim \mathrm{Poisson}( \lambda_{ij} )\quad, \quad 
    \lambda_{ij} = \lambda_0 \exp\left(\frac{1}{\sqrt{m}}
    \bs{u}_i^\intercal \bs{v}_j +  s^u_i + s^v_j \right), \label{eq:z_ij}
\end{align}
where $\lambda_0 = C\Exp(N)$.  

%% file: sections/algorithm.tex
\section{AMP-Based Estimation}
\label{sec:algo}

\subsection{Regularized Maximum Likelihood}
\label{subsec:gaussian}

We consider estimating the parameters \eqref{eq:theta} with the minimization:
\begin{equation}
    \label{eq:thetamin}
    \wh{\theta} = \argmin_{\theta} L_0(\theta),
\end{equation}
where $L_0(\theta)$ is the regularized negative log likelihood:  
\begin{align}  
    L_0(\theta) := 
    -\sum_{ij} \log P_{\rm out}\left(Z_{ij}|\frac{1}{\sqrt{m}}\bs{u}_i^\intercal \bs{v}_j+ s_i^u+ s_j^v \right) + \phi_u(U) + \phi_v(V), \label{eq:loss_poisson}
\end{align}
and $P_{\rm out}(z|\log \lambda) := e^{-\lambda} \lambda^z/z!$ is the Poisson distribution \eqref{eq:z_ij}
and $\phi_u(U)$ and $\phi_v(V)$
are regularizers on the matrices of embedding vectors.  
We will assume the regularizers 
are row-wise separable meaning
\begin{equation} \label{eq:regsep}
    \phi_u(U) = \sum_{i=1}^m g_u(u_i),
    \quad
    \phi_v(V) = \sum_{j=1}^n g_v(v_j),
\end{equation}
for some functions $g_u(\cdot)$
and $g_v(\cdot)$. 
For example, we can use squared norm
regularizers such as:
\begin{align} \label{eq:l2reg}
    g_u(\bs{u}_i) :=\frac{\lambda_u}{2}\|\bs{u}_i\|^2,
    \quad
    g_v(\bs{v}_j) := \frac{\lambda_v}{2}\|\bs{v}_j\|^2,
\end{align}
for normalization constants $\lambda_u$ and $\lambda_v$.  Regularizers can 
also be used to impose sparsity.  %\textcolor{red}%{Provide more reasoning on the regularizers}
Sparsification is especially important when addressing the resource-intensive learning of pre-trained transformers and their applications in Natural Language Processing (e.g. see \cite{jaiswal2023sparsity}).

\subsection{Two step estimation}
The minimization \eqref{eq:thetamin} can be performed in practice through a variety
of methods such as stochastic gradient descent.  
However, these methods are difficult to directly analyze.  We thus consider a simpler to
analyze, but approximate two step method:
\begin{itemize}
    \item First, we estimate the bias terms $s_i^u$ and $s_j^v$ through a simple frequency counting; and
    \item Second, we estimate the embedding vectors through a modification of the low-rank AMP procedure of \cite{lesieur2015mmse,Lesieur_2017}.
\end{itemize}
The next two sub-sections describe each of these steps.

\subsection{Bias vector estimation}\label{sub:bias-est}
% \textcolor{red}{Clean up text to make it flow.}
% Our problem is to estimate the embedding
% vectors 
% $\bs{u}_i$ and $\bs{v}_j$
% and the biases $s_i^u$ and $s_j^v$
% from the model \eqref{eq:pmf}
% and measurements $Z_{ij}$. 
As the first step, we would like to estimate $s_i^u$ and $s_j^v$'s given measurements $Z_{ij}$.
% In principle,
% one can use maximum likelihood (ML) estimation. 
% In this work, we consider a simpler
% two-step estimation method
% that is easier to analyze.
% In the first step, we estimate
% the bias terms $s_i^u$ and $s_j^v$.
Define:
\begin{equation} \label{eq:ruv_def}
    r_i^u := e^{s_i^u}, 
    \quad 
    r_j^v := e^{s_j^v}.
\end{equation}
% \textcolor{red}{I know this is a big change, but can you define $r_i^u = e^{s_i^u}$ in the entire document?
% As I look through this, it is a pain to keep using the inverse.  Also, then $r_i^u$ is basically the 
% relative frequency, not its inverse.}\textcolor{blue}{can we keep the definition for $r_u$ and $r_v$ as $e^{-s}$? Changing that will change all my calculations. Also, the example I found on Zipf matches better with 1/r for relative frequency}
% \textcolor{red}{You do not need to re-run the simulations.  We just change the notation in the
% paper.  The code can still use the inverse.}
% Our plan is to estimate $r_i^u$
% and $r_j^v$ and then compute estimates of $s_i^u$
% and $s_j^v$ from \eqref{eq:ruv_def}.
Note that, by 
adjusting the bias terms $s_i^u$
or $s_j^v$, we will assume in the sequel, 
without loss of generality, that
in the model \eqref{eq:z_ij}
\begin{equation} \label{eq:bias_norm}
    \lambda_0 = 1, 
    \quad
    \frac{1}{m}\sum_{i=1}^m r_i^u = 1.
\end{equation}
Under the above assumption, we propose to estimate the bias terms with:
\begin{equation} \label{eq:sest}
    \wh{s}_i^u = \log(\wh{r}_i^u),
    \quad
     \wh{s}_j^v = \log(\wh{r}_j^v),
\end{equation}
where $\wh{r}_i^u$ and $\wh{r}_j^v$ are
estimates of $r_i^u$ and $r_j^v$ given by:
\begin{equation} \label{eq:rest}
    \wh{r}_i^u = \frac{m}{Z_{\rm tot}} \sum_{j=1}^n Z_{ij},
    \quad
     \wh{r}_j^v = 
     \frac{n}{Z_{\rm tot}} \sum_{i=1}^m Z_{ij} 
\end{equation}
and
\begin{equation}
     Z_{\rm tot} := \sum_{i=1}^m \sum_{j=1}^n
    Z_{ij}.
\end{equation}  
%\textcolor{red}{Give some intuition on this formula.  Basicially, it is the fraction of occurrences of the each $X_1 = i$ and $X_2 = j$.}
We note that based on \eqref{eq:rest} and \eqref{eq:ruv_def}, $e^{s_i^u}$ is proportional to the fraction of times $X_1=i$ occurs in the given samples. A similar argument holds for $e^{s_j^v}$ and frequency of $X_2=j$. 
% Note that $Z_{\rm tot}$ 
% represents the total number of measurements.
% Also, 
% $1/\wh{r}_i^u$ is proportional to 
% $\sum_j Z_{ij}$, which is simply
% the relative frequency of the event $X_1=i$.
% Similarly, $1/\wh{r}_j^v$ 
% is the relative frequency 
% of the event $X_2=j$. 
% %A more detailed calculation is provided in Appendix section \ref{app:biases}.
% \begin{proof}
% \begin{align*}
%     \mathbb{E}|\{t|(x^t_1 = i)\}| &= \sum_{j=1}^n \mathbb{E}(Z_{ij}) = \sum_{j=1}^n \lambda_{ij} \nonumber\\&= \exp(s^u_i)\sum_{j=1}^n \exp(O(\frac{1}{\sqrt{m}}))\exp(s^v_j)\nonumber\\&\approx \exp(s^u_i)\sum_{j=1}^n \exp(s^v_j)\\
%     \mathbb{E} (N) &= \sum_{i=1}^m\sum_{j=1}^n \lambda_{ij}\nonumber\\&\approx \sum_{i=1}^m\exp(s^u_i)\sum_{j=1}^n \exp(s^v_j)\\
%     \mathbb{P}(X_1=i) &\approx \frac{\exp(s^u_i)}{\sum_{i=1}^m\exp(s^u_i)} 
% \end{align*}
% We can estimate the bias values by solving:
% \begin{align*}
%     \frac{\exp(s^u_i)}{\sum_{i=1}^m\exp(s^u_i)} = \frac{\sum_{j=1}^n Z_{ij}}{\sum_{i=1}^m \sum_{j=1}^n Z_{ij}} \quad \forall i \in [m]
% \end{align*}
%  The bias terms $s^v_j$'s can be estimated with a similar process. 
% \end{proof}
% Of course, the loss function 
% \eqref{eq:loss_poisson} 
% depends on the bias terms $s_i^u$
% and $s_j^v$, which are not known.
% In place, we will use the estimates
% $\wh{s}_i^u$ and $\wh{s}_j^v$
% from \eqref{eq:sest}.
\subsection{Biased Low-Rank AMP}
%\textcolor{red}{Clean this up and state that $L_0(U,V)$ is $J(\theta)$ with the bias terms fixed.}
Ideally, having bias estimates $\wh{\bs{s}}^u$ and $\wh{\bs{s}}^v$ from the previous step, we would obtain estimates for $U$ and $V$ by minimizing:
\begin{equation} \label{eq:uvmin}
    \wh{U}, \wh{V} = \arg \min_{U,V} 
    L_0(U,V,\wh{\bs{s}}^u, \wh{\bs{s}}^v),
\end{equation}
where $L_0(\cdot)$ is the negative log likelihood in 
\eqref{eq:loss_poisson}.  To simplify
the notation, we will sometimes drop the dependence
on $\wh{\bs{s}}^u$ and  $\wh{\bs{s}}^v$, and write:
\begin{align} 
    \MoveEqLeft L_0(U,V) := 
    -\sum_{ij} \log P_{\rm out}\left(Z_{ij}|\frac{1}{\sqrt{m}}\bs{u}_i^\intercal \bs{v}_j + \wh{s}_i^u+\wh{s}_j^v\right) + \phi_u(U) + \phi_v(V).
        \label{eq:loss_prior}
\end{align}

% % One possible approach to minimizing the loss
% % \eqref{eq:loss_poisson} 
% An initial approach is to use the low-rank AMP method of \cite{lesieur2015mmse,Lesieur_2017} to minimize the loss \eqref{eq:loss_poisson} using estimates $\wh{s}_i^u$ and $\wh{s}_j^v$.  
% %This method  considers general loss functions of the form
% The method \cite{lesieur2015mmse,Lesieur_2017} permits general probability mappings
% $P(Z_{ij}|\cdot)$ (these are called 
% the ``output channels").  However, there
% is no direct method to incorporate the bias
% terms $s_i^u$ and $s_j^v$ that appear 
% in the likelihood in \eqref{eq:loss_poisson}.
To solve the minimization \eqref{eq:loss_prior}, one
could attempt to use prior AMP literature such as \cite{guionnet2023spectral,mergny2024fundamentallimit}.  However,
in  \eqref{eq:loss_prior}, the bias terms $\wh{s}_i^u$ and $\wh{s}_j^v$
create a dependence on the output channel $P_{\rm out}(\cdot)$ 
with the indices $i$ and $j$.  This dependence is not considered
in the prior works.
We thus propose the following modification of the low-rank AMP method in
\cite{lesieur2015mmse,Lesieur_2017}.
The low-rank AMP method \cite{lesieur2015mmse,Lesieur_2017}
takes a quadratic approximation of
the log likelihood of the output channel.
We apply a similar approach here and 
first compute the so-called Fisher score
functions:
\begin{align}
    \MoveEqLeft 
    Y_{ij} :=  \left.\frac{\partial}{\partial w} 
     \log P_{\rm out}(Z_{ij}| w+s_i^u+s_j^v)\right|_{w=0} = \frac{1}{r_i^ur_j^v}\left( Z_{ij} - r_i^ur_j^v \right). \label{eq:score_ij}
\end{align}
Also, let $\Delta_{ij}$ denote the so-called
inverse Fisher information:
\begin{align}
    \MoveEqLeft 
    \frac{1}{\Delta_{ij}} := %-\left.\frac{\partial^2}{\partial u^2} \Exp\left[ \log P_{\rm out}(Z_{ij}|u+s_i^u+s_j^v) \right] \right|_{u=0}   \nonumber \\
    \Exp\left[ \left(\left.\frac{\partial }{\partial w} \log P_{\rm out}(Z_{ij}|w+s_i^u+s_j^v)\right|_{w=0} \right)^2  \right]
     =r_i^u r_j^v
     \label{eq:del_ij}
\end{align}
Next, let $M_{ij} := (\bs{u}_i^\intercal \bs{v}_j)/\sqrt{m}$.  For large $m$, $M_{ij}$ is small, so we can take a Taylor's
approximation,
\begin{align}
     \MoveEqLeft \log P_{\rm out}(Z_{ij} | M_{ij} + s_i^u + s_j^v) \approx Y_{ij}M_{ij} - \frac{1}{2\Delta_{ij} }M_{ij}^2 + \mathrm{const}.
    \label{eq:logquad1}
\end{align}
To write this as a quadratic, define
the scaled variables:
\begin{equation} \label{eq:abyscaled}
    A := R_u^{1/2}U, \quad 
    B := R_v^{1/2}V, \quad 
    \tl{Y} := R_u^{1/2}Y R_v^{1/2}, 
\end{equation}
where $R_u$ and $R_v$ are diagonal matrices with diagonal elements $r_i^u$'s and $r_j^v$'s, respectively.
Then, using \eqref{eq:score_ij},
\eqref{eq:del_ij}, \eqref{eq:logquad1}
and some simple algebra shows
that the log likelihood can be written in 
a quadratic form:
\begin{align}
     \MoveEqLeft -\log P_{\rm out}(Z_{ij} | M_{ij} + s_i^u + s_j^v) \approx \frac{1}{2}\left|\tl{Y}_{ij} - \frac{1}{\sqrt{m}}[AB^\intercal]_{ij} \right|^2 + \mathrm{const}.
    \label{eq:logquad2}
\end{align}
Hence, we can approximate the loss function
\eqref{eq:loss_poisson} as:
\begin{align} 
L_0(U,V) &\approx L(A,B) + \mathrm{const},
\end{align}
where
\begin{align}
    \MoveEqLeft L(A,B) :=\frac{1}{2}\left\|  \tl{Y}- \frac{1}{\sqrt{m}}AB^\intercal \right\|^2_F +\phi_u(R_u^{-1/2}A) + \phi_v(R_v^{-1/2}B), 
\label{eq:LAB}
\end{align}
and then find the minima:
\begin{equation} \label{eq:ABmin}
    \wh{A},\wh{B} = \argmin_{A,B} L(A,B).
\end{equation}
We call $L(A,B)$ the \emph{quadratic
approximate loss function}.

\begin{algorithm}[t]
\caption{Biased Low Rank AMP}\label{alg:low-ramp}
\begin{algorithmic}[1]
\REQUIRE Number of iterations $K_{it}$; 
denoisers $G_a(\cdot)$, $G_b(\cdot)$; initial matrix $\wh{B}_0 \in \mathbb{R}^{n\times d}$; observation matrix $Z$
\STATE Estimate $\{\wh{r}^u_i,\wh{r}^v_j\}$ using \eqref{eq:rest}  
\STATE Compute $\tl{Y}$ using bias estimates and \eqref{eq:score_ij},\eqref{eq:abyscaled}
\STATE Initialize $k = 0$, $\Gamma_k^a=0$
\WHILE{$k<K_{it}$}

\STATE $F_k^a = \frac{1}{m}\wh{B}_k^\intercal\wh{B}_k - \Gamma^a_k$
\STATE $P_k^a = \frac{1}{\sqrt{m}} \tl{Y}\wh{B}_k - \wh{A}_{k-1} \Gamma^a_k$
\STATE $[\wh{A}_k]_{i\ast}  =  G_a([P^a_k]_{i\ast},\wh{r}^u_i, F_k^a) \quad \forall i\in[m]$  \label{step:ahat} %=  [P^a_k]_{i\ast} (F_k^a + \lambda_u r^u_i I)^{-1} \quad \forall i\in[m] $
%\STATE $\Gamma_k^b = \frac{1}{m}\sum_{i=1}^m (F_k^a + \lambda_u r^u_i I)^{-1}$\\
\STATE  $\Gamma^b_k = \frac{1}{m}\sum_{i=1}^m \partial G_a([P^a_k]_{i\ast},\wh{r}^u_i,F^a_k) /\partial [P^a_k]_{i\ast}^{\intercal}$ 
\STATE $F^b_k = \frac{1}{m}\wh{A}_k^\intercal \wh{A}_k - \Gamma^b_k$
\STATE $P_k^b = \frac{1}{\sqrt{m}} \tl{Y}^\intercal \wh{A}_k - \wh{B}_k \Gamma^b_k$
\STATE $[\wh{B}_{k+1}]_{j\ast} = G_b([P_k^b]_{j\ast},\wh{r}^v_j, F_k^b) \quad \forall j\in[n] $ \label{step:bhat}%= [P^b]_{i\ast} (F_k^b + \lambda_v r^v_i I)^{-1} \quad \forall i\in[n]$
\STATE  $\Gamma^a_{k+1} = \frac{1}{n}\sum_{j=1}^n \partial G_b([P^b_k]_{j\ast},\wh{r}^v_j,F_k^b) /\partial [P^b_k]_{j\ast}^{\intercal}$ 

\STATE $k \gets k+1$
\ENDWHILE
\STATE return $\wh{A}_k$ and $\wh{B}_{k+1}$
\end{algorithmic}
\end{algorithm}

To solve the  minimization
\eqref{eq:ABmin},
we consider a generalization
of the rank one method of \cite{fletcher2018iterative}
and \cite{lesieur2015mmse} 
shown in Algorithm~\ref{alg:low-ramp},
which we call \emph{Biased Low-Rank AMP}.
Here, the function $G_a(\cdot)$ is the denoiser
\begin{align} 
    G_a(P^a,R_u, F^a) &:= 
    \argmin_A  -\tr [ (P^a)^\intercal A ] + \frac{1}{2} \tr[ F^a A^\intercal A ]  
     + \phi_u(R_u^{-1/2}A)  \label{eq:Gadef}
\end{align}
which in the row-wise form simplifies to:
\begin{equation} \label{eq:Gadef-row}
    G_a(p_i,r_i^u, F^a) := \argmin_a \frac{1}{2} a^\intercal F^a a - p_i^\intercal a + g_u(\frac{1}{\sqrt{r_i^u}} a)
    %\argmin_a \frac{1}{2} \|a-p_i\|^2_{F^a} + g_u(\sqrt{r_i^u} a),
\end{equation}
%where we use the notation $\|x\|^2_F = x^\intercal F x$.
The denoiser $G_b(\cdot)$ is defined similarly.
The updates for the $\Gamma^a_k$ and $\Gamma^b_k$ are:
\begin{subequations}\label{eq:Gamma_abdef}
    \begin{align}
        \Gamma^a_k &= \frac{1}{n}\sum_{j=1}^n \frac{\partial G_b([P^b_k]_{j\ast},r^v_j,F^b_k)}{\partial [P^b_k]_{j\ast}^{\intercal}}  \\
        \Gamma^b_k &= \frac{1}{m}\sum_{i=1}^m \frac{\partial G_a([P^a_k]_{i\ast},r^u_i,F^a_k)}{\partial [P^a_k]_{i\ast}^{\intercal}}.
    \end{align}
\end{subequations}
Algorithm~\ref{alg:low-ramp} is identical to the low-rank AMP
algorithm of \cite{lesieur2015mmse,Lesieur_2017} but with two key differences:
First, and most importantly, the denoisers in steps \ref{step:ahat} and \ref{step:bhat}
in Algorithm~\ref{alg:low-ramp} have bias terms $\wh{r}_j^u$ and $\wh{r}^v_j$.
In the low-rank AMP algorithm \cite{lesieur2015mmse,Lesieur_2017}, the denoisers are the same for all rows.
In this sense, one key contribution of this work is to show that the embedding estimation
with variability in the term frequencies can be accounted for by a 
variable denoiser.  We will also show below that the state evolution analysis 
of the algorithm can be extended.

A second, and more minor difference, is that the  low-rank AMP algorithm of 
 \cite{lesieur2015mmse,Lesieur_2017} considers only MMSE denoisers.  Here, our analysis
 will apply to arbitrary Lipschitz denoisers.  In particular, the simulations below
 consider denoisers with a minimization \eqref{eq:Gadef} similar to the so-called
 MAP estimation in the AMP literature.
%\textcolor{red}{In Algorithm, add the estimate of the bias terms and hats to the $r$.}

\subsection{Fixed Points}

As a first convergence result, the following Lemma shows that if the algorithm
converges, its fixed point is, at least, a local minimum of the objective. %The proof is provided in Appendix section \ref{subsec:lemma-proof}.  
\begin{lemma}
Any fixed point of Algorithm \ref{alg:low-ramp} is a local minimum of \eqref{eq:LAB}.
\label{lem:fixed-points} 
\end{lemma}

\begin{proof}
Consider any fixed point of Algorithm \ref{alg:low-ramp}.  We drop the dependence on the iteration $k$.  
Then, the minimizer $\wh{A}$ satisfies: 
\begin{align}
    \MoveEqLeft \wh{A} =  G_a( P^a, R_u, F^a) \nonumber \\
    & \stackrel{(a)}{\Rightarrow} 
    R_u^{-1/2}\phi_u'(R_u^{-1/2}\wh{A}) - P^a + \wh{A}F^a = 0 \nonumber \\
    & \stackrel{(b)}{\Rightarrow}  R_u^{-1/2}\phi_u'(R_u^{-1/2}\wh{A}) - \frac{1}{\sqrt{m}} \tl{Y}\wh{B} + \wh{A}\Gamma^a 
    + \wh{A}F^a = 0 \nonumber \\
    & \stackrel{(c)}{\Rightarrow}  R_u^{-1/2}\phi_u'(R_u^{-1/2}\wh{A}) - \frac{1}{\sqrt{m}} \tl{Y}\wh{B}  
    + \frac{1}{m} \wh{A}\wh{B}^\intercal \wh{B} = 0 \nonumber \\
    & \stackrel{(d)}{\Rightarrow}  \frac{\partial L(A,\wh{B})}{\partial A} = 0,
\end{align}
where (a) follows from taking the derivative of the objective function of the denoiser in \eqref{eq:Gadef};
(b) follows from the update of $P_k$; 
(c)  follows from the update of $F^a_k$;
and (d) follows from taking derivative of the objective function \eqref{eq:LAB}.  Similarly, we can show that $\partial L(\wh{A},B)/\partial B = 0$.  Hence, $(\wh{A},\wh{B})$ is 
a critical point of \eqref{eq:LAB}.
\end{proof}
% For the squared norm reguarlizer 
% \eqref{eq:l2reg}, it can be verified that
% the denoisers are given by:
% \begin{subequations}
% \begin{align}
%    G_a([P^a_k]_{i\ast},r^u_i,F^a_k) &= [P^a_k]_{i\ast}(F^a_k+\lambda_u r^u_i I_d)^{-1}\\ 
%   G_b([P_k^b]_{j\ast},r^v_j, F_k^b) &= [P_k^b]_{j\ast} (F_k^b + \lambda_v r^v_j I_d)^{-1}   
% \end{align}    
% \end{subequations}

%% file: sections/results.tex
\section{Analysis in the Large System Limit}
\label{sec:results}
\subsection{Formal model}
The benefit of the AMP method is that
the performance of the algorithm
can be precisely analyzed in a certain
\emph{large system limit} (LSL) as is
commonly used
in studying AMP algorithms. 
In the LSL,
we consider a sequence of problems
indexed by $n$.  For each $n$,
we assume that $m = m(n)$ where
\begin{equation}\label{eq:beta}
    \lim_{n \rightarrow \infty} \frac{m(n)}{n} = \beta,
\end{equation}
for some $\beta > 0$. 
That is, the number of
values of the random variables $X_1$ and $X_2$ grow linearly.   Importantly, the embedding dimension
$d$ remains fixed.

Next, we assume that the bias terms
$r_i^u$ and $r_j^v$ as well as the 
true embedding vectors $\bs{u}_i$ and $\bs{v}_j$
have a certain limiting distribution.
Specifically, recall that the rows
of the matrices $A$ and $B$ in \eqref{eq:abyscaled} are the scaled
true embedding vectors:
\[
    [A]_{i*} = \sqrt{r_i^u} \bs{u}_i, 
    \quad
    [B]_{j*} = \sqrt{r_j^v} \bs{v}_j. 
\]
Similarly, the rows of $\wh{A}_0$ and $\wh{B}_0$
are the initial estimates of the rows
of $A$ and $B$.  We assume these quantities are
deterministic, 
but converge empirically with second-order moments (see~Definition \ref{def:pseudolipschitz}
for a precise definition of the concept)
to random variables
\begin{subequations}  \label{eq:suvlim}
\begin{align} 
    & \{r_i^u, [A]_{i\ast}, [\wh{A}_0]_{i\ast}\}_{i=1}^m \stackrel{PL(2)}{\longrightarrow} (R^u, \mc{A}, \wh{\mc{A}}_0), 
    \\
    &\{r_j^u, [B]_{j\ast}, [\wh{B}_0]_{j\ast}\}_{j=1}^n \stackrel{PL(2)}{\longrightarrow} (R^u, \mc{B}, \wh{\mc{B}}_0), 
\end{align}
\end{subequations}
where $R^u$ and $R^v$ are scalar
random variables and $\mc{A}$, $\mc{B}$,
$\mc{\wh{A}}_0$, and $\mc{\wh{B}}_0$
are random $d$-dimensional vectors.
One particular case where the convergence
\eqref{eq:suvlim} occurs is that 
values $\{r_i^u\}$, $\{r_j^v\}$ 
are drawn i.i.d.\ from
$R^u$ and $R^v$ respectively, and
$([A]_{i\ast}, [\wh{A}_0]_{i\ast})$,
$([B]_{i\ast}, [\wh{B}_0]_{i\ast})$,  are drawn i.i.d.\ from
$(\mc{A},\wh{\mc{A}}_0)$ and $(\mc{B},\wh{\mc{B}}_0)$ respectively.
Note that we have used the caligraphic
letters such as $\mc{A}$ and $\mc{B}$
to denote the random variables
describing the distribution of the rows
of the matrices $A$ and $B$.

As a second and critical simplifying assumption,
let 
\begin{equation} \label{eq:Wdef}
    W = \tl{Y} - \frac{1}{\sqrt{m}} AB^\intercal.
\end{equation}
For given $r_i^u$ and $r_j^v$, 
using the fact that $Z_{ij}$ are i.i.d.,
Poisson random variables with 
distribution \eqref{eq:z_ij},
it can be shown that $W_{ij}$ are i.i.d.,
with mean and second moments:
\begin{equation}
    \lim_{n \rightarrow \infty} \Exp(W_{ij}) = 0,
    \quad
    \lim_{n \rightarrow \infty} \Exp(W_{ij}^2) = 1,
\end{equation}
To simplify the analysis, we will approximate
$W_{ij}$ as Gaussian.  That is,
we will assume that 
$\tl{Y}$ is generated from
\begin{equation} \label{eq:YGauss}
    \tl{Y} = \frac{1}{\sqrt{m}} AB^\intercal
    + W, \quad W_{ij} \sim {\mc{N}}(0,1).
\end{equation}

Finally, 
we assume that the random variables and vectors in \eqref{eq:suvlim} are bounded
and $G_a(\cdot)$ and $G_b(\cdot)$
are Lipschitz continuous.

\subsection{Selecting the bias distribution} 

The above formal probabilistic model for the 
variables allows us to capture 
key attributes of the parameters by correctly selecting the random variables.
We first start by discussing how to select the distributions of $R^u$ and $R^v$.
The variables $R^u$ and $R^v$ model
the variability in the bias terms, which
in turn can model the variability in the marginal
distributions of the terms.  
As an example, 
consider the following:
%\textcolor{red}{Clean up}
It is well known that the distribution of word occurances in human language roughly obeys a power law, namely Zipf's law, where the $\ell$-th most frequent term has a frequency proportional to $\frac{1}{\ell^\alpha}$ for $\alpha\approx 1$ \cite{Piantadosi2014-qt}. %We note a similar relation between biases and term frequencies in \eqref{eq:rest}.
Suppose we want to model the terms coming from a Zipf law. Specifically, suppose $X_1 \in \{1,\ldots,m\}$ represents the index
for one of $m$ terms and the term probabilities are given by Zipf Law:
\[
    P(X_1=i) = \frac{C_m}{i^{\alpha}}
\]
for some constant $C_m$.
From \eqref{eq:marginal} we know that $r_i^u = c_1 P(X_1=i)$ for some constant $c_1$. Without loss of generality assume that $c_1=1$. Then,
\[
    r_i^u = \frac{C_m}{i^{\alpha}}.
\]
Since $C_m$ is arbitrary, we can take $C_m = C_0 m^\alpha$ for some $C_0$, 
so
\[
    r_i^u = C_0 (i/m)^{-\alpha}.
\]
It can be easily verified that:
\begin{equation} \label{eq:Ruzipf}
    \{ r_i^u \} \plconv  R^u := \frac{C_0}{U^\alpha}, \quad U \sim \mathrm{Unif}[0,1],
\end{equation}
where $\mathrm{Unif}$ denotes the uniform distribution. Hence, by selecting $R^u$ as in \eqref{eq:Ruzipf}, we can capture a Zipf distribution.
Other distributions are also possible.

\subsection{Selecting the embedding vector distributions}
For the embedding vectors,
the distribution of $\mc{A}$ and $\mc{B}$
can capture structural properties
of the embeddings.  
% \textcolor{red}{These properties
% can include features such as
% norm constraints, or sparsity.
% See if there are any references
% for using sparse embeddings. \cite{liang2021anchor}}
These properties
can include features such as
norm constraints, or sparsity.
As an example, sparse interdependent representation of words is especially beneficial for large vocabularies due to training, storage, and inference concerns that arise in large language models \cite{liang2021anchor}. 

Finally, the model can also capture the number of samples:
Let $N = \sum_{ij} Z_{ij}$ denote the total number of
training samples, so $N/(nm)$ is the number of samples
per pair of unknowns $(i,j)$ in the probability of the event, $(X_1,X_2)=(i,j)$.  This number of samples scales as:
\begin{align}
   \lim_{n,m\rightarrow \infty} \frac{N}{nm} &=
   \lim_{n,m\rightarrow \infty} 
   \frac{1}{nm}\sum_{ij}Z_{ij} \nonumber \\
   &\stackrel{(a)}{=} 
   \lim_{n\rightarrow \infty} 
        \frac{\lambda_0}{nm} \sum_{ij} \exp\left( s_i^u + s_j^v \right)
        \nonumber \\
     &\stackrel{(b)}{=} \lim_{n,m\rightarrow \infty} 
        \lambda_0 \left(\sum_{i} \frac{r_i^u}{m}\right) 
        \left(\sum_{j} \frac{r_j^v}{n}\right) 
        \nonumber \\  
    &\stackrel{(c)}{=}\lambda_0 \Exp(R^u)\Exp(R^v),
\end{align}
where, in step (a), we have used \eqref{eq:z_ij}
along with the fact that the $1/\sqrt{m}$ can be ignored in the limit; 
step (b) follows from the definitions of $r_i^u$ and $r_j^v$
in \eqref{eq:ruv_def},
and step (c) follows from the assumption of empirical
convergence \eqref{eq:suvlim}.  
The assumption \eqref{eq:bias_norm} requires that $\lambda_0 = 1$
and $\Exp(R^u) = 1$.  In this case, $\Exp(R^v)$
controls the total number of samples per unknown.
By adjusting this scaling we can thus analyze the sample
complexity of the estimation.

\subsection{Main results}
Our main result shows that,
under the above assumptions,
the joint distribution of true embedding vectors and
their estimates can be exactly predicted
by a state evolution (SE). The SE,
shown in Algorithm~\ref{alg:se} is a modification of the result in
\cite{fletcher2017expectation}.
The SE
generates a sequence of deterministic
quantities such as $\bar{M}_k^a$, $\bar{Q}_k^a$, $\bar{F}_k^a$,
as well as random vectors such as $\mc{P}_k^a$ and $\wh{\mc{A}}_k$. 
% \textcolor{blue}{in SE M and Q need matrices to be computed???} 
% \textcolor{red}{Good point.  To be precise, in the SE, let's use the variables
% $\bar{M}_k$, $\bar{F}_k$ and $\bar{Q}_k$.  Then we show that $M_k \rightarrow \bar{M}_k$.
% Can you change the text.
% }
%Detailed computations of SE for the problem at hand are provided in Appendix section \ref{app:se}.
%Now, we are ready to state our main theorem summarizing these results:
%\textcolor{red}{make this a precise theorem statement.}
\begin{theorem}\label{theorem1}
Under the above assumptions, consider the outputs of Algorithm~\ref{alg:low-ramp}
and the state evolution updates in Algorithm~\ref{alg:se}.
Then, for every $k$
\begin{subequations} \label{eq:param-conv}
    \begin{align}
    &\lim_{n \rightarrow \infty} (M_k^a,F_k^a,Q_k^a) = 
    (\bar{M}_k^a,\bar{F}_k^a,\bar{Q}_k^a), \\ 
&\lim_{n \rightarrow \infty} (M_k^b,F_k^b,Q_k^b) = 
    (\bar{M}_k^b,\bar{F}_k^b,\bar{Q}_k^b), 
\end{align}
\end{subequations}
where the convergence is almost surely and the quantities on the left
are from  Algorithm~\ref{alg:low-ramp}
and the quantities from the right are from SE Algorithm~\ref{alg:se}.
In addition, 
the joint distributions of the embedding vectors and their estimates converge as 
\begin{subequations} \label{eq:empiricalA}
    \begin{align}
    &([A]_{i\ast}, [\wh{A}_k]_{i\ast}, r_i^u, \wh{r}_i^u)
    \plconv (\mc{A}, \wh{\mc{A}}_k, R^u, R^u) \\
    &([B]_{j\ast}, [\wh{B}_k]_{j\ast}, r_j^v, \wh{r}_j^v)
    \plconv (\mc{B}, \wh{\mc{B}}_k, R^v, R^v)
\end{align}
\end{subequations}
\end{theorem}

To understand the result first consider the convergence of the 
bias terms $r_i^u$ and their estimates $\wh{r}_i^u$.
The results show that the estimates are asymptotically consistent.  For example,
the empirical convergence $PL(2)$ implies: 
\[
    \lim_{n\rightarrow \infty} \frac{1}{n} \sum_{i=1}^n |r_i^u - \wh{r}_i^u|^2 = \Exp|R^u - R^u|^2 = 0.
\]
The convergence result also enables us to compute error metrics on 
the estimated embedding vector.  For example, using $PL(2)$ convergence, we can compute the average
MSE on each row as:
\begin{align}
    \MoveEqLeft \lim_{n \rightarrow \infty} \frac{1}{n} \sum_{i=1}^n \|\left[ A\right]_{i*} - 
    [\wh{A}_k]_{i*}\|^2 =  \Exp \|\mc{A} - \wh{\mc{A}}_k\|^2_2,
\end{align}
where the right-hand side can be evaluated using the distributions of the random variables from the SE.
We can also evaluate quantities such as the overlap:
\begin{align}
    \MoveEqLeft \lim_{n \rightarrow \infty} \frac{1}{n} \sum_{i=1}^n |[A]_{i*}^\intercal 
    [\wh{A}_k]_{i*}| =  \Exp |\mc{A}^\intercal \wh{\mc{A}}_k|,
\end{align}
or any other similar metric.  
Importantly, we can also see how this MSE varies with the relative frequency.
For example, the quantity
\[
    \Exp \left( \|\mc{A} - \wh{\mc{A}}_k\|^2_2 \mid  R^u = r \right),
\]
describes the MSE as a function of the term frequency $r$.  Thus, we can see, for example,
how well the estimator performs on terms that occur infrequently.
% describes how the MSE varies with the relative frequency of the term.

% %It is argued, in section \ref{sec:lsl-proof} that 
% From these distributions, we can then 
% compute any row-wise metrics on
% the error of the estimated embedding 
% vectors -- see \cite{fletcher2018iterative}
% for examples. 

% For instance, the SE results allow us to derive exact formulation for the MSE. We will indeed see how MSE varries with the bias values (i.e. relative frequency of terms) in simulations of Section \ref{subsec:mse-fisher}.
% % \textcolor{red}{Provide an example.
% % Say that the result thus allows
% % us to see how an error such as MSE
% % varies with the $R_u$, the relative
% % frequency of the terms. 
% % We will indeed see this in this
% % simulation...}

\begin{algorithm}[t]
\caption{State Evolution}\label{alg:se}
\begin{algorithmic}[1]
\REQUIRE Number of iterations $K_{it}$;
denoisers $G_a(\cdot)$, $G_b(\cdot)$;
initial random row vector $\wh{\mc{B}}_0 \in \Real^d$.
\STATE Initialize $k = 0$, $\Gamma_k^a=0$
\WHILE {$k<K_{it}$}
\STATE $\bar{M}_k^b = \Exp( \mc{B}^\intercal \wh{\mc{B}}_k)$,
$\bar{Q}_k^b = \Exp( \wh{\mc{B}}_k^\intercal\wh{\mc{B}}_k)$
\STATE $\bar{F}_k^a = \bar{Q}_k^b - \Gamma_k^a$
\STATE $\mc{P}_k^a = \mc{A}\bar{M}_k^b + \mathcal{N}(0,\bar{Q}_k^b)$
\STATE $\wh{\mc{A}}_k = G_a(\mc{P}_k^a, R^u, \bar{F}_k^a)$ 
\STATE $\Gamma_k^b = \Exp\left[ \partial G_a(\mc{P}_k^a, R^u, \bar{F}_k^a) / \partial \mc{P}_k^a \right]$
\STATE $\bar{M}_k^a = \Exp( \mc{A}^\intercal \wh{\mc{A}}_k)$,
$\bar{Q}_k^a = \Exp( \wh{\mc{A}}_k^\intercal\wh{\mc{A}}_k)$
\STATE $\bar{F}_k^b = \bar{Q}_k^a - \Gamma_k^b$
\STATE $\mc{P}_k^b = \mc{B}\bar{M}_k^a + \mathcal{N}(0,\bar{Q}_k^a)$
\STATE $\wh{\mc{B}}_{k+1} = G_b(\mc{P}_k^b, R^v, \bar{F}_k^b)$ 
\STATE $\Gamma_{k+1}^a = \Exp\left[ \partial G_a(\mc{P}_k^b, R^v, \bar{F}_k^b) / \partial \mc{P}_k^b \right]$

\STATE $k \gets k+1$
\ENDWHILE
\STATE return $\wh{\mc{A}}_k$ and $\wh{\mc{B}}_{k+1}$
\end{algorithmic}
\end{algorithm}

%% file: sections/lsl_proofs.tex
\section{Proofs} \label{sec:lsl-proof}
\subsection{Preliminaries}
We begin with the following technical definitions on convergence \cite{emami2020}:
\begin{definition}\label{def:pseudolipschitz}
    (Pseudo-Lipschitz continuity). For a given $p\geq 1$, a function $\phi: \mathbb{R}^\ell \rightarrow \mathbb{R}^r$ is called Pseudo-Lipschitz continuous if for some constant $C>0$ we have:
    \[
    \|\phi(x_1) - \phi(x_2)\| \leq C\|x_1-x_2\|(1+\|x_1\|^{p-1}+\|x_2\|^{p-1})
    \]
\end{definition}
\begin{definition}\label{def:pl(p)}
    (Empirical convergence of a sequence) Consider a sequence $\{x_i\}_{i=1}^n$ with $x_i \in \mathbb{R}^\ell$. For a finite $p\geq 1$, we say that the sequence $\{x_i\}_{i=1}^n$ converges empirically with $p$-th order moments if there exists a random variable $X\in \mathbb{R}^\ell$ such that:
    \begin{enumerate}
        \item $\mathbb{E}(\|X\|_p^p)< \infty$
        \item For any $\phi: \mathbb{R}^\ell \rightarrow  \mathbb{R}$ that is pseudo-Lipschitz continuous of order $p$, \[
        \lim_{n\rightarrow \infty} \frac{1}{n}\sum_{i=1}^n \phi(x_i) = \mathbb{E}[\phi(X)].
        \]
    \end{enumerate}
\end{definition}
When $\{x_i\}_{i=1}^n$ converges empirically to $X$ with 
$p$-th order moments, we will write:
    \[
    \lim_{n\rightarrow \infty}\{x_i\}_{i=1}^n \pl X
    \]
    We note that $PL(p)$ convergence is also equivalent to convergence in Wasserstein-$p$ metric \cite{Villani2008-pw}. 
For the theorems below, we will focus on the case
when $p=2$.
Also, when the context is clear, we may simply write
$x_i\plconv X$ instead of 
$\{x_i\}_{i=1}^n \plconv X$.
We also need the following formulae for a Poisson random variable.

\begin{lemma}  \label{lem:poisson}  Let $X$ be a Poisson random variable with $\Exp(X) = \lambda$.
Then, the second and fourth central moments are (\cite{kendall1946}):
\begin{equation}
    \Exp(X-\lambda)^2 = \lambda, \quad
    \Exp(X-\lambda)^4 = \lambda + 3\lambda^2.
\end{equation}    
\end{lemma}

We next need a simple bound on the square
of sums of random variables:
\begin{lemma} \label{lem:sumsq}  
Let $x_{ik}$, $i=1,\ldots,n$,$k=1,\ldots,K$, be a set of scalars. Then,
\[
    \sum_{i=1}^n \left| \sum_{k=1}^K x_{ik} \right|^2 \leq K^2 \max_k 
     \sum_{i=1}^n |x_{ik} |^2.
\]
\end{lemma}
\begin{proof} Let $M = \max_k \sum_{i=1}^n |x_{ik}|^2$.
Then,  
\begin{align*}
    \MoveEqLeft  \sum_{i=1}^n |\sum_{k=1}^K x_{ik} |^2 
    = \sum_{k=1}^K\sum_{\ell=1}^K\sum_{i=1}^n x_{ik} \bar{x}_{i\ell} \leq \sum_{k=1}^K\sum_{\ell=1}^K \left| 
    \sum_{i=1}^n x_{ik}\bar{x}_{i\ell} \right|      \leq K^2 M,
\end{align*}
where $\bar{x}_{i\ell}$ denotes the conjugate of $x_{i\ell}$ and the last step follows from Cauchy-Schwartz.
\end{proof}
We will also use the following variant of the strong law of large numbers (SLLN).
Recall that a variable $Y$ is \emph{uniformly bounded by} a variable $X$ if
\begin{equation}
    P(|Y| \geq t) \leq P(|X| \geq t)
\end{equation}
for all $t \geq 0$.
\begin{lemma}[SLLN for triangular arrays, Theorem 2 of \cite{hu1989strong}]  \label{lem:slln} Let $X_{ni}$, $i=1,\ldots,n$, $n=1,2,\ldots$
be a triangular array of zero-mean, independent random variables that is uniformly bounded by a random variable $X$ with $\Exp(X^{2p}) < \infty$ for $1\leq p <2$.
Then, 
\begin{equation}
    S_n = \frac{1}{n^{1/p}} \sum_{i=1}^n X_{ni} \rightarrow 0
\end{equation}    
almost surely.
\end{lemma}

\begin{lemma} \label{lem:poissonlim}
Suppose that $P_n \sim \mathrm{Poisson}(\lambda_n)$ are independent with  $\Exp(P_n)/n = \lambda_n/n \rightarrow \lambda$.  Then, $P_n/n \rightarrow \lambda$
almost surely.
\end{lemma}
\begin{proof}  Since $P_n \sim \mathrm{Poisson}(\lambda_n)$, we can write 
\[
    P_n = \sum_{i=1}^n Y_{ni}, \quad Y_{ni} \sim \mathrm{Poisson}(\lambda_n/n).
\]
Let $X_{ni} = Y_{ni} - \lambda_n/n$ so $\Exp(X_{ni}) = 0$.
Since $\lambda_n/n \rightarrow \lambda$, it can be verified that $X_{ni}$
is uniformly bounded by a random variable with $\Exp|X|< \infty$.
Therefore,
\begin{align}
   \lim_{n \rightarrow \infty}  \frac{P_n}{n} - \lambda 
   =\lim_{n \rightarrow \infty} \frac{1}{n}\left[ P_n - \lambda_n \right] = \lim_{n \rightarrow \infty} \frac{1}{n} \sum_{i=1}^n (Y_{ni} - \lambda_n/n)  =  \lim_{n \rightarrow \infty} \frac{1}{n} \sum_{i=1}^n X_{ni} = 0,
\end{align}
where we have used Lemma~\ref{lem:slln} and the convergence is almost surely.
\end{proof}

\subsection{Consistency of the Estimates of the Bias Terms}
We first prove the convergence of the bias terms.
\label{sub:consistency}
\begin{lemma}  \label{lem:biasconv}  Under the assumptions of Section~\ref{sec:results},
the biases $r_i^u$ and their corresponding estimates $\wh{r}^u_i$ converge
empirically to:
\begin{subequations}
\begin{align}
    \lim_{n \rightarrow \infty} \{(r_i^u, \wh{r}^u_i)\}_{i=1}^{m} \stackrel{PL(2)}{=} (R^u,R^u) 
    \label{eq:rulim} \\
    \lim_{n \rightarrow \infty} \{(r_j^v, \wh{r}^v_j)\}_{j=1}^{n} \stackrel{PL(2)}{=} (R^v,R^v) 
    \label{eq:rvlim} 
\end{align}
\end{subequations}
\end{lemma}

\begin{proof}  We will prove \eqref{eq:rulim}; the proof of \eqref{eq:rvlim} is similar.  Also,
to be clear, we will use $r_{ni}^u$ and $\wh{r}_{ni}^u$
for $r_i^u$ and $\wh{r}_i^u$ to make the dependence on $n$
in these quantities explicit.
Fix any $PL(2)$ function $\phi(r,\wh{r})$.
We need to show 
\begin{equation} \label{eq:philimrr}
    \lim_{n \rightarrow \infty} \frac{1}{m(n)} \sum_{i=1}^{m(n)} \phi(r_{ni}^u, \wh{r}_{ni}^u) = \Exp(\phi(R^u, R^u)).
\end{equation}
From the assumption \eqref{eq:suvlim}, 
we know $\{r_{ni}^u\}_{i=1}^m \plconv R^u$, and therefore:
\begin{align}
    \MoveEqLeft \lim_{n \rightarrow \infty} \frac{1}{m(n)} \sum_{i=1}^{m(n)} \phi(r_{ni}^u, \wh{r}_{ni}^u) = \Exp(\phi(R^u, R^u)) 
     + \lim_{n \rightarrow \infty} \frac{1}{m(n)} \sum_{i=1}^{m(n)} \left[ \phi(r_{ni}^u, \wh{r}_{ni}^u) 
     -\phi(r_{ni}^u, r_{ni}^u)  \right].
\end{align}
Since $\phi(\cdot)$ is $PL(2)$, to prove \eqref{eq:philimrr}, it suffices to show:
\begin{align} \label{eq:rrlim}
    \MoveEqLeft \lim_{n \rightarrow \infty} \frac{1}{m} \sum_{i=1}^{m}(r_{ni}^u - \wh{r}_{ni}^u)^2  = 0,
\end{align}
almost surely.  In \eqref{eq:rrlim},
we have dropped the dependence of $m$
on $n$ to simplify the notation.
From \eqref{eq:rest}, we can  
write the estimate $\wh{r}^u_{ni}$ as a fraction:
\begin{equation} \label{eq:rhatratio}
    \wh{r}^u_{ni} = \frac{A_{ni}}{B_n}, \quad A_{ni} = \frac{A_{ni}'}{n} 
    \quad 
    B_n = \frac{B_n'}{nm} 
\end{equation}
and
\begin{equation} \label{eq:ABprimedef}
    A_{ni}' =  \sum_{j=1}^n Z_{ij},
    \quad 
    B_n' = \sum_{i=1}^m \sum_{j=1}^n Z_{ij}.
\end{equation}
Therefore, to prove \eqref{eq:rrlim},
we need to show
\begin{align} \label{eq:rrlim1}
   \lim_{n \rightarrow \infty} \frac{S_n}{B_n^2} = 0,
\end{align}
where
\begin{equation} \label{eq:Sndef}
   S_n = \frac{1}{m} \sum_{i=1}^{m} \epsilon_{ni}^2,
   \quad \epsilon_{ni} := 
    B_nr_{ni}^u - A_{ni}.
\end{equation}
We will prove \eqref{eq:rrlim1} by showing
\begin{equation} \label{eq:BSlim}
    \lim_{n\rightarrow \infty} B_n = \Exp(R^v) \quad , \quad \lim_{n\rightarrow \infty} S_n = 0
\end{equation}
almost surely.
From \eqref{eq:z_ij}, the expectation of
$Z_{ij}$ is:
\begin{align} \label{eq:expzij}
    \Exp(Z_{ij}) &= \lambda_0\exp(s^u_i + s^v_j) + O(1/\sqrt{m}) = r_i^u r_j^v + O(1/\sqrt{m}),
\end{align}
where, in the second step, we used
\eqref{eq:ruv_def} and the assumption
\eqref{eq:bias_norm} that $\lambda_0 = 1$.
Also, since the variables $Z_{ij}$
are independent Poisson random variables,
$A_{ni}'$ and $B_n'$ in \eqref{eq:ABprimedef}
are Poisson random variables with expectations:
\begin{equation}
    \Exp(A_{ni}') = n \Exp(A_{ni})\quad,\quad
    \Exp(B_n') = nm \Exp(B_n) 
\end{equation}
The limit of these expectations are:
\begin{align}
    \lim_{n \rightarrow \infty} 
    \Exp(B_n) &\stackrel{(a)}{=} \lim_{n \rightarrow \infty} \frac{1}{nm} \sum_{i=1}^m\sum_{j=1}^n  \Exp(Z_{ij}) \nonumber \\
    &\stackrel{(b)}{=} \lim_{n \rightarrow \infty} \left(\frac{1}{m} \sum_{i=1}^m r_i^u
    \right) \left(\frac{1}{n} \sum_{j=1}^n r_j^v
    \right) \nonumber \\
    &\stackrel{(c)}{=} \lim_{n \rightarrow \infty} \frac{1}{n} \sum_{j=1}^n r_j^v
    \stackrel{(d)}{=} \Exp(R^v)
    \label{eq:Bnlim}
\end{align}
where the convergence is almost surely 
and (a) follows from \eqref{eq:ABprimedef};
(b) follows from \eqref{eq:expzij}; (c) follows from the
normalization assumption \eqref{eq:bias_norm}; and
(d) follows from the $PL(2)$ convergence
assumption \eqref{eq:rrlim}.
Similarly,
\begin{align}
    \lim_{n \rightarrow \infty} \frac{\Exp(A_{ni})}{r_i^u}
    &=\lim_{n \rightarrow \infty} \frac{1}{n} 
    \frac{1}{r_i^u} \sum_{j=1}^n  \Exp(Z_{ij}) 
    \stackrel{(a)}{=} \lim_{n \rightarrow \infty} \frac{1}{n} \sum_{j=1}^n r_j^v  
    \stackrel{(b)}{=} \Exp(R^v) \label{eq:Anlim}
\end{align}
where, again (a) follows from \eqref{eq:expzij}
and (b) follows from the $PL(2)$ convergence
assumption \eqref{eq:rrlim}.
Since $B_n = B_n'/(nm)$
and $B_n'$ is Poisson, \eqref{eq:Bnlim} and Lemma~\ref{lem:poissonlim}
show that 
\begin{equation} \label{eq:Bnlim1}
    B_n \rightarrow  \Exp(R^v)
\end{equation}
almost surely.  The limit \eqref{eq:Bnlim1} is the first
of the two limits in \eqref{eq:BSlim} that we need to show.
Next, we show that $S_n \rightarrow 0$ almost surely; that is,
we show the second limit in \eqref{eq:BSlim}.
To this end,
write the error terms $\epsilon_{ni}$ as  a sum of four terms:
\begin{equation}
    \epsilon_{ni} = \sum_{k=1}^4 \epsilon_{ni}^{(k)},
\end{equation}
where
\begin{subequations}
    \begin{align}
        \epsilon_{ni}^{(1)} &:= r_{ni}^u(B_n - \Exp(B_n)) \\
        \epsilon_{ni}^{(2)} &:= r_{ni}^u(\Exp(B_n) - \Exp(R^v)) \\
        \epsilon_{ni}^{(3)} &:= \Exp(A_{ni})-A_{ni} \\
        \epsilon_{ni}^{(4)} &:= r_{ni}^u(\Exp(R^v)-\frac{\Exp(A_{ni})}{r_{ni}^u} )
    \end{align}
\end{subequations}
Hence, if we define:
\begin{equation}
    S_n^{(k)}= \frac{1}{m} \sum_{i=1}^{m} (\epsilon_{ni}^{(k)})^2,
\end{equation}
Lemma~\ref{lem:sumsq} shows that
\begin{equation}
    S_n \leq 4^2 \max_{k=1,\ldots,4} S_n^{(k)}.
\end{equation}
Therefore, we can show that $S_n \rightarrow 0$ almost surely if 
\begin{equation} \label{eq:Snklim}
    \lim_{n\rightarrow \infty} S_n^{(k)} = 0 \mbox{ for all }
    k =1,\ldots,4
\end{equation}
almost surely.  
We prove \eqref{eq:Snklim} for the cases $k=1$ and $k=2$. The other two are proven in a similar manner.
For $k=1$:  
\begin{align}
     S_n^{(1)} &= \frac{1}{m} \sum_{i=1}^{m} (\epsilon_{ni}^{(1)})^2 \leq \bar{r}^2_{\rm max}(B_n - \Exp(B_n))^2  
\end{align}

Let $Y_{n} = (B_n - \Exp(B_n))^2$ so we need to show that
$Y_n \rightarrow 0$ almost surely.
% and $X_{n} = Y_{n} - \Exp(Y_{n})$. It is easy to verify that $X_{n}$ is uniformly bounded by $Y_{n}$. Also, we note that $\Exp(X_{n})=0$.
% Then
% \begin{align}
%      S_n^{(1)} &\leq \bar{r}^2_{\rm max}\left[ X_{n} + \Exp(Y_{n}) \right]%= \frac{1}{m} \sum_{i=1}^{m} X_{ni} + \Exp(Y_{ni})
% \end{align}
From \eqref{eq:rhatratio}, we have:
\[
    Y_{n} = \frac{1}{(nm)^2} (B_n' - \Exp(B_n'))^2
\]
Since $B_n'$ in \eqref{eq:ABprimedef} is a Poisson random variable with $\Exp(B_n')=O(mn)$, Lemma \ref{lem:poisson} shows    :
%\begin{subequations}
\begin{align} \label{eq:Ynsq}
   %\Exp Y_n = \frac{1}{(mn)^2}\Exp[(B_n' - \Exp(B_n'))]^2 &=O(1/m^{1.5}n)\\
   \Exp (Y_n^2) =  \frac{1}{(mn)^4}\Exp[(B_n' - \Exp(B_n'))]^4&= O\left(\frac{1}{m^2n^2}\right) 
\end{align}  
%\end{subequations}
For any $\delta>0$, Chebyshev inequality gives:
\begin{align}
    \mathbb{P}(|Y_n|\geq \delta)&\leq \frac{\Exp(Y_{n}^2)}{\delta^2}
\end{align}
Therefore, from \eqref{eq:Ynsq},
\begin{equation}
    \sum_n \mathbb{P}(|Y_n|\geq \delta) = \frac{1}{\delta^2} 
        \sum_n O\left(\frac{1}{m^2n^2}\right) < \infty.
\end{equation}
So, by the Borel-Cantelli lemma \cite{mileBorelLesPD,cantelli1917}, the
event that $P(|Y_n| \geq \delta)$ can occur only finitely many times.
Since this is true for all $\delta$, $Y_n \rightarrow 0$ almost surely.

% Now we can apply Lemma \ref{lem:slln} to get $\lim_{n\rightarrow\infty}S_n^{(1)} = 0$.

For $k=2$:
\begin{align}
 \Exp(B_n) &=  \frac{1}{nm} \sum_{i=1}^m\sum_{j=1}^n  \Exp(Z_{ij}) = \Exp(R^v) + O(1/\sqrt{m})
\end{align}
So,
\begin{align}
 (\Exp(B_n)- \Exp(R^v))^2  =  O(1/m).
\end{align}
Hence,
\begin{align}
     S_n^{(2)} &= \frac{1}{m} \sum_{i=1}^{m} (\epsilon_{ni}^{(2)})^2 
     \leq \bar{r}^2_{\rm max}\frac{1}{m} \sum_{i=1}^{m} (\Exp(B_n) - \Exp(R^v))^2 = O(1/m).     
\end{align}
This gives $\lim_{n\rightarrow\infty}S_n^{(2)} = 0$.
% For $k=3$:\\
% \begin{align}
%     S_n^{(3)} = \frac{1}{m} \sum_{i=1}^{m} X_{ni} + \Exp(Y_{ni})
% \end{align}
% where 
% \[
% Y_{ni} := (A_{ni}-\Exp(A_{ni}))^2 = \frac{1}{n^2} (A_{ni}'-\Exp(A_{ni}'))^2
% \]
% and
% \[
% X_{ni} = Y_{ni} - \Exp(Y_{ni})
% \]
% We know that $A_{ni}'$ is Poisson with $\lambda=O(n)$, hence using Lemma \ref{lem:poisson}:
% \begin{subequations}
%     \begin{align}
%         \frac{1}{n^2}\Exp[(A_{ni}'-\Exp(A_{ni}'))^2] &= O(1/n)\\
%         \frac{1}{n^4}\Exp[(A_{ni}'-\Exp(A_{ni}'))^4] &= O(1/n^2)
%     \end{align}
% \end{subequations}
% similar to the case $k=1$ we conclude $\lim_{n\rightarrow\infty}S_n^{(3)}=0$.
% For $k=4$:
% \begin{align}
%      S_n^{(4)} &= \frac{1}{m} \sum_{i=1}^{m} (\epsilon_{ni}^{(4)})^2 \nonumber \\
%      & \leq \bar{r}^2_{\rm max}\frac{1}{m} \sum_{i=1}^{m} (\Exp(R^v) - \frac{\Exp(A_{ni})}{r_{ni}^u})^2=O(1/m).
% \end{align}
% This gives $\lim_{n\rightarrow\infty}S_n^{(4)} = 0$.
Having proven \eqref{eq:Snklim} for $k=1,2,3,4$ we can then apply the strong law 
of large numbers to show that
$S_n$ in \eqref{eq:Sndef}
converges as $S_n \rightarrow  0$ almost surely. 
\end{proof}

\subsection{Vector-Valued Bayati-Montanari Recursion}
\label{sub:bayati}
In order to prove Theorem \ref{theorem1},
we next need a vector-valued generalization of the Bayati-Montanari recursions \cite{bayati2011dynamics}.
Consider a sequence of recursions, indexed by $n$.  For each $n$,
let $m=m(n)$ satisfying \eqref{eq:beta} for some $\beta > 0$
Let $W \in \mathbb{R}^{n \times m}$ be an i.i.d.\ Gaussian matrix
with entries $W_{ij} \sim {\mathcal N}(0,1)$.
For $k=0,1,\ldots$, consider a general recursion of the form:
\begin{subequations} \label{eq:bmgen}
\begin{align}
    T_k &= \frac{1}{\sqrt{m}}W\wh{B}_k + \wh{A}_{k-1}\Psi^u_k, 
    \\ [\wh{A}_{k}]_{i\ast} &= H_u([T_k]_{i\ast}, Z^u_i,\theta^u_k),\\
    S_k &= \frac{1}{\sqrt{m}}W^\intercal \wh{A}_{k} + \wh{B}_k\Psi^v_k, \\ [\wh{B}_{k+1}]_{j\ast} &= H_v( [S_k]_{j\ast}, Z^v_j, \theta^v_k),
\end{align}
\end{subequations}
which generates a sequence of sets of matrices $(\wh{A}_k,\wh{B}_k,T_k,S_k)$ 
for $k=0,1,\ldots$ with dimensions:
\begin{equation}
    \wh{A}_k, ~T_k \in \mathbb{R}^{m \times d},
    \quad
    \wh{B}_k,~ S_k \in \mathbb{R}^{n \times d},
\end{equation}
for some fixed dimension $d$
(i.e., $d$ does not vary with $n$).
Here, $Z_i^u$ and $Z_j^v$ are variables
that do not change with the index $k$ and 
$H_u(.)$, $H_v(.)$ are functions that are Lipschitz
continuous with Lipschitz continuous derivatives
that operate on the rows
of $T_k$ and $S_k$.  The parameters $\theta^u_k$
and $\theta^v_k$ are assumed to follow updates of the form: 
\begin{subequations}
\begin{align}\label{eq:theta-general}
    \theta^u_k &= \frac{1}{n} \sum_{j=1}^n
        \phi_u([B_k]_{j\ast}, Z_j^v), \\
    \theta^v_k &= \frac{1}{m} \sum_{i=1}^m
    \phi_v([A_k]_{i\ast}, Z_i^u), 
\end{align}
\end{subequations}
for any pseudo-Lipschitz 
continuous functions $\phi_u(\cdot)$
and $\phi_v(\cdot)$.  
Also,
\begin{subequations}
\begin{align}
    \Psi^v_k &= -\frac{1}{m}\sum_{i=1}^m \partial H_u( [T_k]_{i\ast},Z^u_i,\theta^u_k) /\partial [T_k]_{i\ast}^{\intercal}\\
    \Psi^u_{k} &= -\frac{1}{n}\sum_{j=1}^n \partial H_v([S_k]_{j\ast},Z_j^v,\theta^v_k) /\partial [S_k]_{j\ast}^{\intercal}
\end{align}
\end{subequations}
% Comparing these to the recursion considered in \cite{fletcher2018iterative} (equations A.2-A.4), we note that this is the matrix generalization of the Bayati-Montanari recursions.
% As in our problem, consider a sequence of 
% Let $U \in \mathbb{R}^{n \times d}$ and $V\in \mathbb{R}^{n \times d}$
% whose rows 
Assume that parameters $Z^u_i$ and 
$Z^v_j$ and the rows of the initial conditions $\wh{A}_0$ and $\wh{B}_0$ converge as:
\begin{subequations}  \label{eq:bminitlim}
\begin{align} 
    & \{([\wh{A}_0]_{i\ast}, Z^u_i)\}_{i=1}^m \stackrel{PL(2)}{\longrightarrow} (\mc{A}_0,
    \mc{Z}^u), 
    \\
     & \{[(\wh{B}_0]_{i\ast}, Z^v_j)\}_{j=1}^n \stackrel{PL(2)}{\longrightarrow} (\mc{B}_0,
    \mc{Z}^v),
\end{align}
\end{subequations}
for some random vectors $\mc{A}_0$,
$\mc{B}_0$, $\mc{Z}^u$, and $\mc{Z}^v$.
% \textcolor{red}{[Write the SE for $\mc{A}_k$
% and $\mc{B}_k$.]}
Define:
%\begin{subequations}
\begin{align}
     \bar{\theta}^u_k := \Exp(\phi_u(\mc{B}_k,\mc{Z}^v))\quad   
      \bar{\theta}^v_k := \Exp(\phi_v(\mc{A}_k,\mc{Z}^u))  
\end{align}
%\end{subequations}
where $\mc{A}_k$ and $\mc{B}_k$ for $k=1,2,...$ can be calculated using the SE below:
\begin{subequations}
    \begin{align}
        \mc{T}_k &\sim \mc{N}(0,\mathbb{E}(\mc{B}_k^\intercal \mc{B}_k))\\
        \mc{A}_{k} &= H_u(\mc{T}_k,\mc{Z}^u,\bar{\theta}^u_k)\\
        \mc{S}_k &\sim \mc{N}(0,\mathbb{E}(\mc{A}_{k}^\intercal \mc{A}_{k}))\\
        \mc{B}_{k+1} &= H_u(\mc{S}_k,\mc{Z}^v,\bar{\theta}^v_k)
    \end{align}
\end{subequations}

\begin{theorem} \label{thm:bm}
Under the above assumptions, for any fixed iteration
$k$,
\begin{equation} \label{eq:bm-theta}
    \lim_{n\rightarrow \infty} \theta_k^u = \bar{\theta}_k^u,
    \quad
     \lim_{n\rightarrow \infty} \theta_k^v = \bar{\theta}_k^v,
\end{equation}
almost surely and
\begin{subequations} \label{eq:bm-plconv}
  \begin{align}
     \lim_{n \rightarrow \infty}
     \{ ([\wh{A}_{k}]_i, Z^u_i) \}
    &= (\mc{A}_{k}, \mc{Z}^u ) \\
     \lim_{n \rightarrow \infty} \{ ([\wh{B}_{k+1}]_{j\ast},Z^v_j) \}
    &= (\mc{B}_{k+1}, \mc{Z}^v) 
\end{align}  
\end{subequations}
where the convergence is $PL(2)$.
\end{theorem}

\begin{proof}  The result for the case $d=1$
was proven in the original work by 
Bayati and Montanari \cite{bayati2011dynamics}.
An extension to the matrix-valued case (i.e., $d > 1$)
can be found in \cite{pandit2021matrix}.
The works \cite{bayati2011dynamics} and
\cite{pandit2021matrix} 
however, do not include the 
data-dependent parameters $\theta^u_k$ and $\theta^v_k$.  The addition of the parameters
can be done along the lines of \cite{kamilov2012approximate}.
\end{proof}

\subsection{Proof of Theorem~\ref{theorem1}}
To apply Theorem~\ref{thm:bm}, we
write Algorithm~\ref{alg:low-ramp}
in the format of \eqref{eq:bmgen}.
Define
\begin{equation} \label{eq:zuvdef}
    Z^u_i = ([A]_{i\ast},r_i^u,\wh{r}_i^u), \quad
    Z^v_j = ([B]_{j\ast},r_j^v,\wh{r}_j^v).
\end{equation}
and 
\begin{equation}\label{eq:theta-specific}
    \theta_k^u =  (M_k^b, F_k^a), \quad
    \theta_k^v =  (M_k^a, F_k^b),
\end{equation}
Assumption \eqref{eq:suvlim} shows
\eqref{eq:bminitlim} is satisfied
if we define the random variables:
\begin{equation} \label{eq:thetarv}
    \mc{Z}^u := (\mc{A},R^u), \quad
    \mc{Z}^v= (\mc{B},R^v).
\end{equation}
Next define:
%\begin{subequations} 
\begin{align}\label{eq:TSdef}
   T_k := P^a_k - A M^b_k  \quad
   S_k := P^b_k - B M^a_k.
\end{align}    
%\end{subequations}
where:
\begin{align}\label{eq:defM}
    M_k^b = \frac{1}{m} B^\intercal \wh{B}_k \quad M_k^a &= \frac{1}{m} A^\intercal \wh{A}_k
\end{align}
%\textcolor{blue}{NOTE: $R^u$ denotes the scalar random variable whereas $R_u$ is the row vector containing $r^u_i$'s.}
We also define the equivalent denoisers as:
\begin{subequations}
    \begin{align}
        H_u([T_k]_{i\ast},Z_i^u,\theta^u_k) &:= G_a([T_k]_{i\ast}+[A]_{i\ast} M^b_k,r^u_i,F^a_k)\\
        H_v([S_k]_{i\ast},Z_j^v, \theta^v_k) &:= G_b([S_k]_{i\ast}+[B]_{i\ast} M^a_k,r^v_j,F^b_k)
    \end{align}
\end{subequations}
and:
%\begin{subequations}
\begin{align}\label{eq:PsiGamma}
    \Psi^v_k := -\Gamma_k^b \quad
    \Psi^u_{k} := -\Gamma_{k}^a 
\end{align}    
%\end{subequations}
Also:
\begin{align}
    T_k &\stackrel{(a)}{=} \frac{1}{\sqrt{m}}\tl{Y}\wh{B}_k - \wh{A}_{k-1}\Gamma_k^a - AM_k^b \nonumber \\
        &\stackrel{(b)}{=} \frac{1}{\sqrt{m}}W\wh{B}_k + \frac{1}{m} A B^\intercal \wh{B}_k- \wh{A}_{k-1}\Gamma_k^a - AM_k^b \nonumber \\
        &\stackrel{(c)}{=} \frac{1}{\sqrt{m}}W\wh{B}_k + A M_k^b - AM_k^b - \wh{A}_{k-1}\Gamma_k^a \nonumber \\
        &\stackrel{(d)}{=} \frac{1}{\sqrt{m}}W\wh{B}_k + \wh{A}_{k-1}\Psi^u_k
\end{align}
where (a) follows from \eqref{eq:TSdef} and the update for $P^a_k$ in Algorithm~\ref{alg:low-ramp}; 
(b) follows from \eqref{eq:YGauss}; (c) follows from the definition of $M_k^b$ in \eqref{eq:defM}, and (d) follows from \eqref{eq:PsiGamma}. Similar arguments can be made for $S_k$.
Finally, from \eqref{eq:theta-specific} observe that 
\begin{align}
    M_k^b &=  \frac{1}{m} \sum_{j=1}^n [B]_{j\ast}^\intercal 
     [\wh{B}_k]_{j\ast} =  \frac{1}{n} \sum_{j=1}^n \frac{1}{\beta}[B]_{j\ast}^\intercal 
     [\wh{B}_k]_{j\ast} \nonumber \\
     F^a_k %&= \frac{1}{n} \sum_{j=1}^n \frac{1}{\beta}[\wh{B}_k]_{j\ast}^\intercal [\wh{B}_k]_{j\ast}-\frac{1}{n}\sum_{j=1}^n \frac{\partial G_b([S_k]_{i\ast}+[B]_{i\ast} M^a_k,r^v_j,F^b_k) }{\partial [S_k]_{j\ast}^{\intercal}} \nonumber \\
     &= \frac{1}{n} \sum_{j=1}^n \left( \frac{1}{\beta}[\wh{B}_k]_{j\ast}^\intercal 
     [\wh{B}_k]_{j\ast}-\frac{\partial H_v([S_k]_{j\ast},Z_j^v,\theta^v_k)}{\partial [S_k]_{j\ast}^{\intercal}} \right)
\end{align}
Hence, the update for $\theta^u_k$ in \eqref{eq:theta-specific} can be written in the form \eqref{eq:theta-general} for appropriate $\phi_u$.  Similarly, $\theta^v_k$
can also be written in the form form \eqref{eq:theta-general} for an appropriate $\phi_v$. Overall, we have
shown that Algorithm~\ref{alg:low-ramp} can be written 
in the form of \eqref{eq:bmgen} and we can apply
Theorem~\ref{thm:bm}.  Then, \eqref{eq:bm-theta}
and \eqref{eq:bm-plconv} show \eqref{eq:param-conv}
and \eqref{eq:empiricalA}, respectively and the proof is complete.

%% file: sections/simulation.tex
\section{Numerical Experiments}
\label{sec:simulations}

\subsection{Denoisers}
\input{sections/examples}
\subsection{Synthetic data}
To validate the SE equations, we first 
consider a simple synthetic data example.
We use $m=2000, n=3000, d=10$ and 
use $L_2$ regularizers \eqref{eq:regl2} with 
 $\lambda_u = \lambda_v = 10^{-3}$. We generate rows of true matrices $U_0$ and $V_0$ following:
 \begin{subequations}\label{eq:sim-true}
  \begin{align}
    \bs{u}_i &\sim \mathcal{N}(0,0.1 I) \quad i \in [m]\\
    \bs{v}_j &\sim \mathcal{N}(0,0.1 I) \quad j \in [n]
\end{align}   
 \end{subequations}

%To generate the problem instance we assume that $s^u_i$'s and $s^v_j$'s randomly take one of the values from the set $\{4.5,5.5\}$.% With this choice we ensure that the average $\Delta$ is below the critical value in \eqref{eq:delta-critical-summary}. 
To generate the problem instance, we assume that $s^u_i$'s and $s^v_j$'s are randomly selected from an exponential distribution with parameter $0.25$. In order to ensure that the average $\Delta$ is below the critical value in \eqref{eq:delta-critical}, we shift all these biases by $5$.
We will use estimations of these biases via \eqref{eq:sest} in our Algorithms. We run Algorithm~\ref{alg:low-ramp} for 20 instances and average our results. The expectations in
the state evolution, Algorithm~\ref{alg:se}, are also
computed with 20 Monte Carlo trials in each iteration.
We initialize the $\wh{A}_k$ and $\wh{B}_k$ matrices with i.i.d.\ entries with zero mean and unit variance Gaussian distributions. 
Fig.~\ref{fig:20runs-loss} shows the loss function \eqref{eq:LAB} (normalized by the true loss) vs iterations, averaged over 20 instances of the problem.  We see that the  average of the loss function observed in the simulations closely matches the predicted training loss from the SE.
% \begin{figure}[t]
%     \centering
%     \includegraphics[width=0.8\linewidth]{figures/edited_L2_20runs_Loss_exp0.25_offset5_m3000_initLR0.1_emb_Adam_decay_0.001lam_final.pdf}
%     \caption{Normalized loss vs iteration averaged over 20 instances, evaluated for an instance of the problem with $m=2000$, $n=3000$, $d=10$, and squared norm regularizers. 
%     }
%     \label{fig:20runs-loss}
% \end{figure}

\begin{figure}[t]%
    \centering
    \subfloat[\centering \label{fig:20runs-loss}]{{\includegraphics[width=0.45\linewidth]{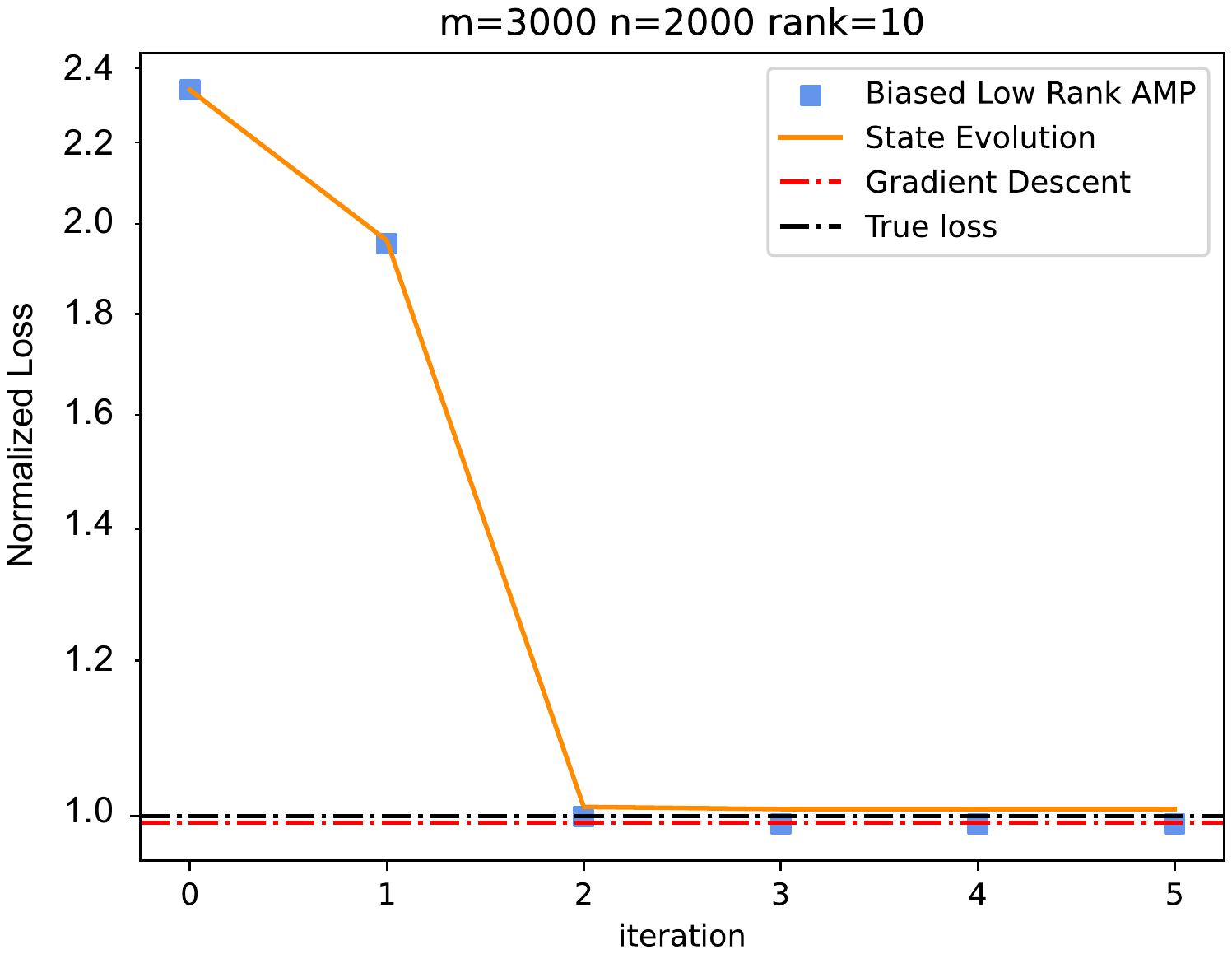} }}%
    \qquad
    \subfloat[\centering \label{fig:mse}]{{\includegraphics[width=0.45\linewidth]{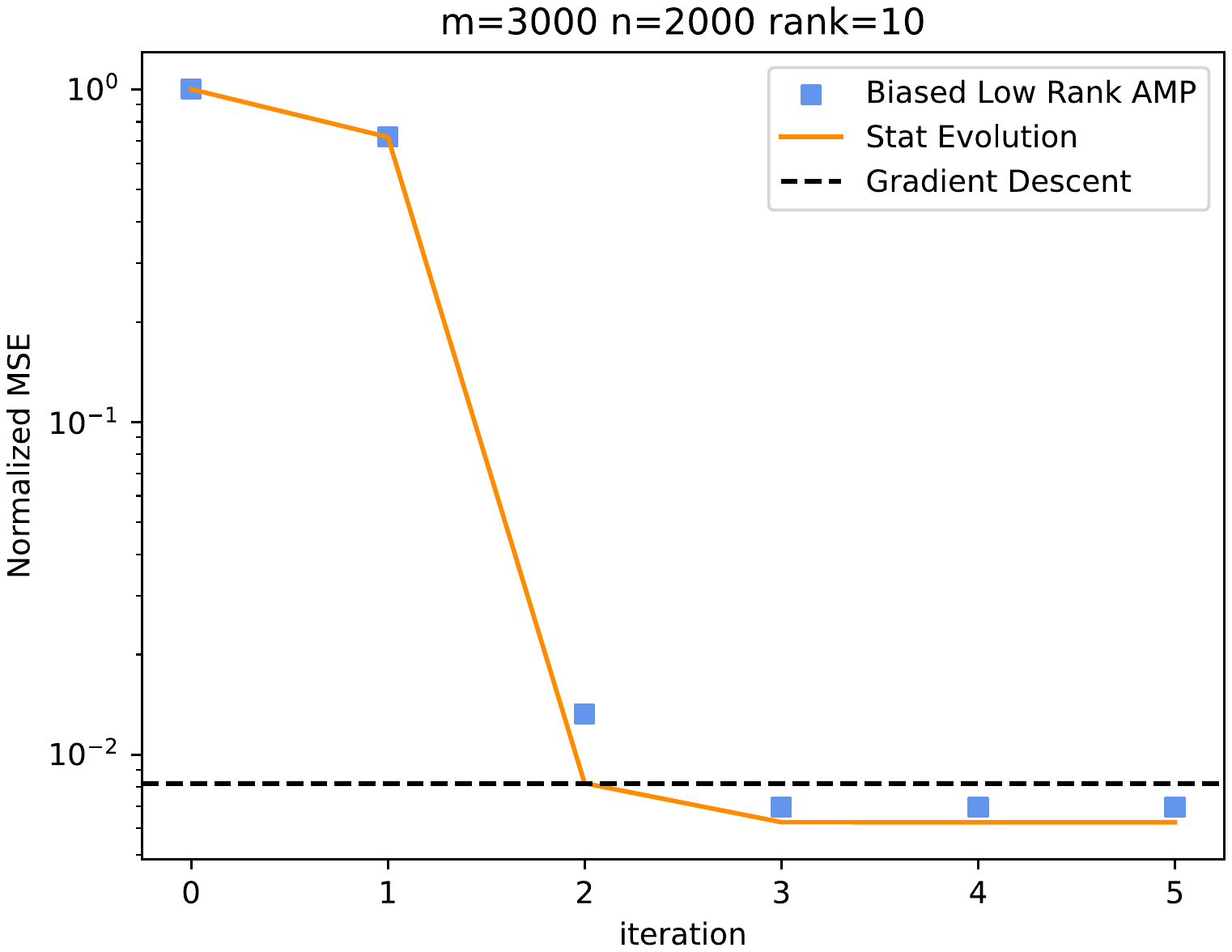} }}%
    \caption{Normalized loss (a) and MSE (b) vs iteration averaged over 20 instances, evaluated for an instance of the problem with $m=2000$, $n=3000$, $d=10$, and squared norm regularizers.}%
    \label{fig:example}%
\end{figure}

We can also use the SE to estimate the  error on the correlation terms:
For each iteration index $k$, let $M_{ij}$ and $\wh{M}^k_{ij}$ denote the true and estimated correlation values:
\begin{equation} \label{eq:mmhat}
    M_{ij} = [A]_{i*}[B]_{j*}^\intercal,
    \quad
    \wh{M}^k_{ij} = [\wh{A}_k]_{i*} [\wh{B}_k]^\intercal_{j*}
\end{equation}
At each iteration $k$, defined  the normalized MSE as:
\begin{equation}
    \textit{MSE}_k := \frac{\mathbb{E}(M_{ij}  - \wh{M}^k_{ij})^2}{\mathbb{E}(M_{ij})^2},
\end{equation}
where the expectation is over the indices
$i$ and $j$. This MSE corresponds to how
well the true correlation of the events $X_1=i$
and $X_2=j$ are predicted.
We can similarly obtain a prediction 
of the MSE from the SE.
Fig.~\ref{fig:mse}
shows the simulated MSE and SE predictions
as a function of the iteration.  Again, we see
an excellent match. The convergence result of applying Gradient Descent (GD) to the same problem is provided in the figures as a reference. Since GD usually takes a few thousands iterations to converge, we have only plotted the final convergence point.  The final error of GD is similar to Biased Low Rank AMP since they both converge to critical points of the loss function. The point is that the performance of the biased Low Rank AMP algorithm  can be exactly predicted with state evolution.
% \begin{figure}[t]
%     \centering
%     \includegraphics[width=0.8\linewidth]{figures/edited_L2_20runs_MSE_exp0.25_offset5_m3000_initLR0.1_emb_Adam_decay_0.001lam_final.pdf}
%     \caption{Normalized MSE vs iteration averaged over 20 instances, evaluated for an instance of the problem with $m=2000$, $n=3000$, $d=10$, and squared norm regularizers. 
%     %We note that $\textit{MSE}=1$ refers to setting $\wh{A}=\textbf{0}, \wh{B}=\textbf{0}$.
%     }
%     \label{fig:mse}
% \end{figure}

We repeat a similar experiment for sparse $U$ and $V$ using regularizers defined in \eqref{eq:regsparse}. To define the sparse matrices we define the rows of matrices similar to \eqref{eq:sim-true} and then randomly set half of the elements in each row to zero. 
%To generate the problem instance we assume that $s^u_i$'s and $s^v_j$'s randomly take one of the values from the set $\{5,6\}$. 
The sampling process of bias terms and selection of all the other parameters are the same as the previous experiment. In order to find the solutions to denoisers \eqref{eq:l1-denoiserA}, we use the Lasso function in the Scikit-learn library (\cite{scikit-learn}) with a warm start to use the solutions of previous iteration as a starting point for the next iteration. Figures \ref{fig:sparse20runs-loss} and \ref{fig:sparse20runs-mse} show the results for sparse regularizers.

\begin{figure}[t]%
    \centering
    \subfloat[\centering \label{fig:sparse20runs-loss}]{{\includegraphics[width=0.45\linewidth]{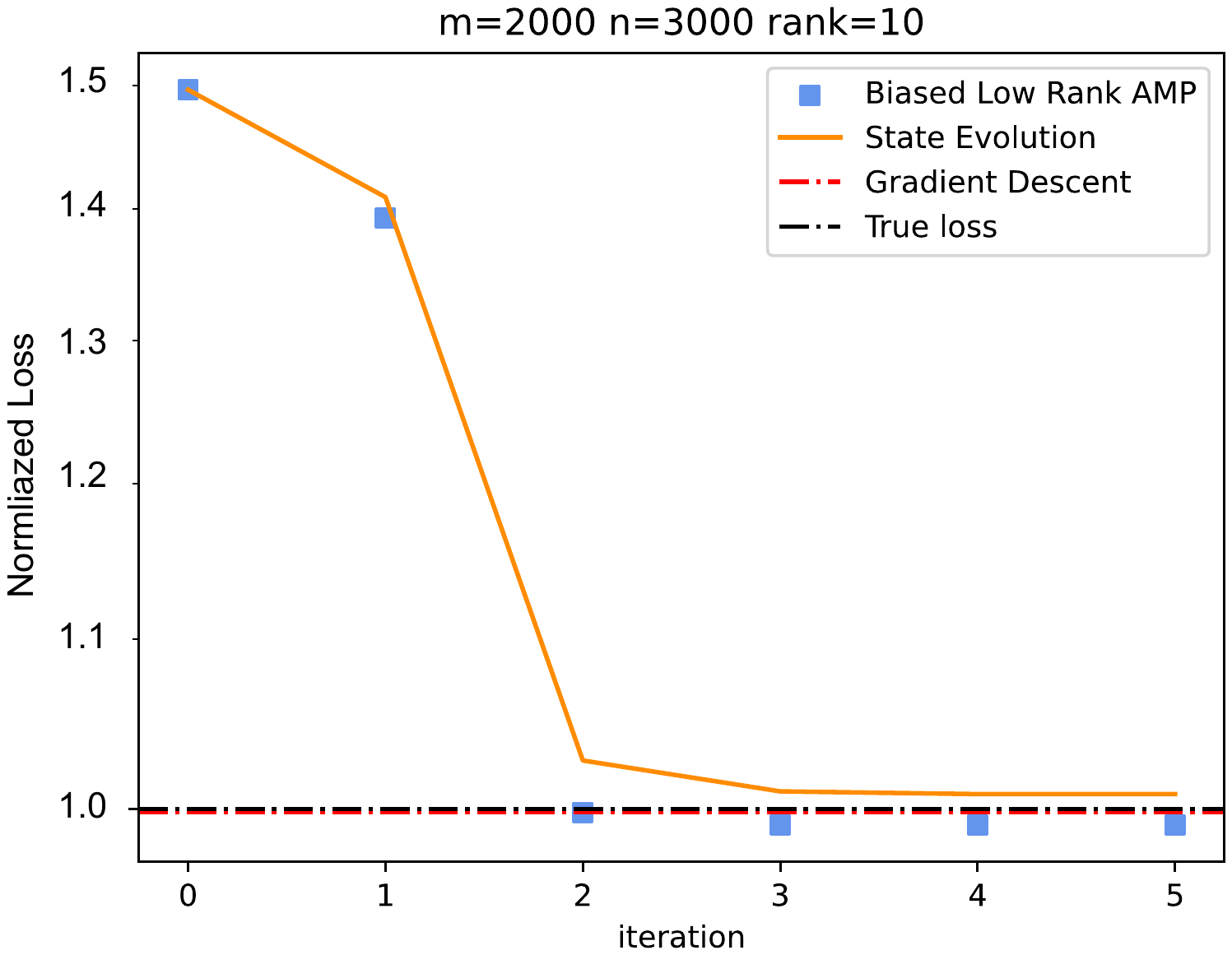} }}%
    \qquad
    \subfloat[\centering \label{fig:sparse20runs-mse}]{{\includegraphics[width=0.45\linewidth]{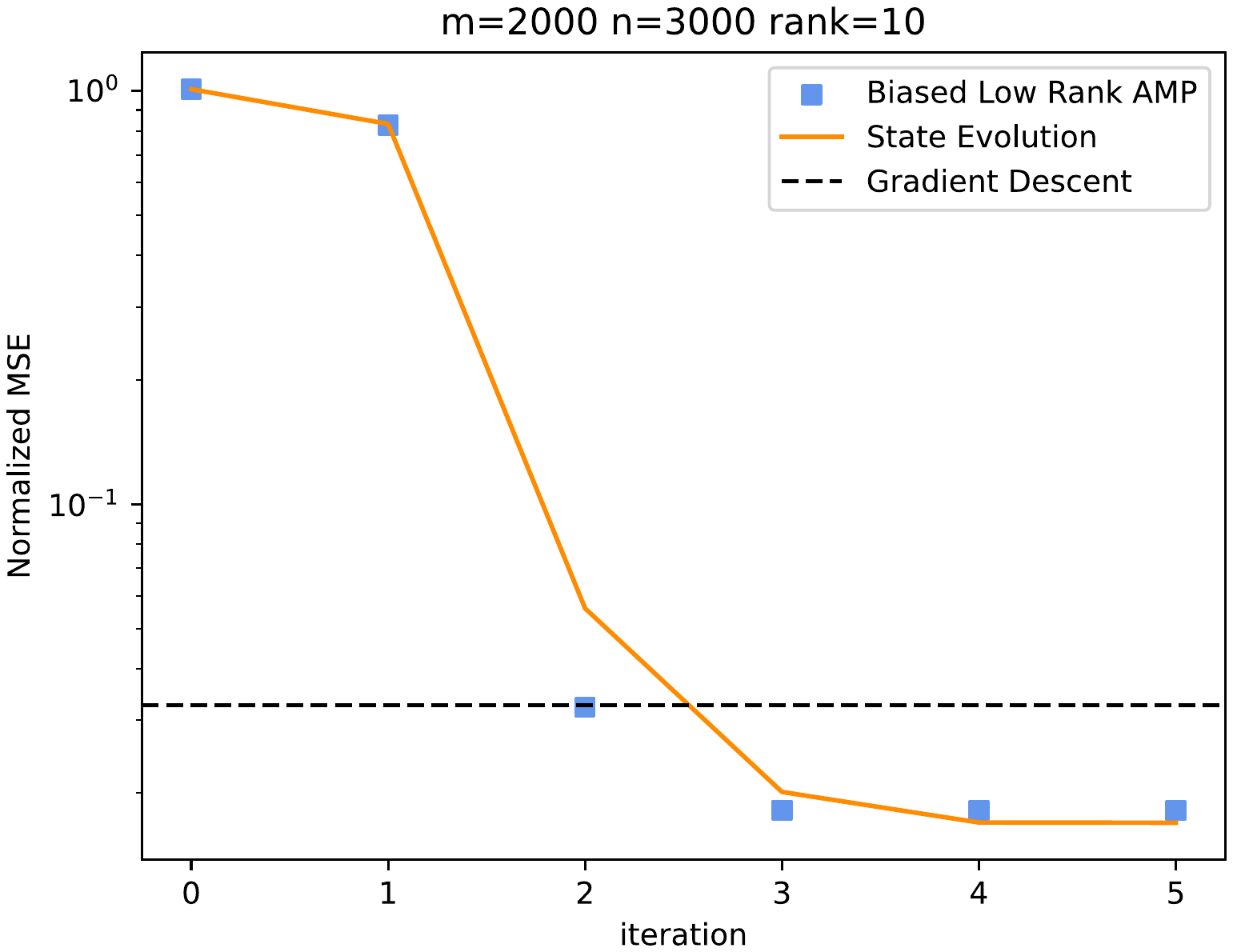} }}%
    \caption{Normalized loss (a) and MSE (b) vs iteration averaged over 20 instances, evaluated for an instance of the problem with $m=2000$, $n=3000$, $d=10$, and L1 norm regularizers.}%
    \label{fig:example2}%
\end{figure}

% \begin{figure}[t]
%     \centering
%     \includegraphics[width=0.8\linewidth]{figures/edited_sparse_20runs_Loss_exp0.25_offset5_final.pdf}
%     \caption{Normalized loss vs iteration averaged over 20 instances, evaluated for an instance of the problem with $m=2000$, $n=3000$, $d=10$, and L1 norm regularizers. 
%     }
%     \label{fig:sparse20runs-loss}
% \end{figure}

% \begin{figure}[t]
%     \centering
%     \includegraphics[width=0.8\linewidth]{figures/edited_sparse_20runs_MSE_exp0.25_offset5_final.pdf}
%     \caption{Normalized MSE vs iteration averaged over 20 instances, evaluated for an instance of the problem with $m=2000$, $n=3000$, $d=10$, and L1 norm regularizers. 
%     }
%     \label{fig:sparse20runs-mse}
% \end{figure}

\subsection{MSE vs. inverse Fisher information}\label{subsec:mse-fisher}
A basic challenge in many text processing problems
is that there is a high variabilty of the 
terms.  In our model, this property is equivalent to variability in the marginal 
probabilities $P(X_1=i)$ and $P(X_2=j)$
over indices $i$ and $j$.  Presumably,
the estimation of the correlation $M_{ij} = \bs{u}_i^\intercal \bs{v}_j$ will be 
better when the $P(X_1=i)$ and $P(X_2=j)$
are higher so that there are more samples
with $(x_1,x_2) = (i,j)$.
This intuition is predicted by our model.
Specifically, the SE reveals that
the key parameter in estimation accuracy 
of $M_{ij}$ is the inverse Fisher information,
$\Delta_{ij}$ in \eqref{eq:del_ij}.
To validate this prediction,
Fig.~\ref{fig:delta-mMSE} shows a scatter plot of samples
of the normalized MSE of $M_{ij}$
vs.\ $\Delta_{ij}$ demonstrating higher
inverse Fisher information results in higher
MSE. The critical value of $\Delta$ (above which spectral algorithms fail) is computed using Marcenko Pastur theorem:
\begin{align}\label{eq:delta-critical}
        \Delta_{\rm critical} &= \frac{\lambda_{\rm max}(\Sigma^u \Sigma^v)}{(1 + \sqrt{\beta})^2}
\end{align}
where $\Sigma^u$ and $\Sigma^v$ are covariance matrices associated with zero-mean distributions $P_u$ and $P_v$ corresponding to $U$ and $V$, respectively. $\lambda_{\rm max}(.)$ is the maximum eigenvalue operator.

We note that the joint distribution of the MSE and Fisher information is well-predicted by the SE. For reference, 
we have also
plotted the results for approximately solving the quadratic minimization \eqref{eq:LAB}  via an SVD of $\tilde{Y}$,
which also matches the biased low-rank AMP in this case.
% \begin{figure}[t]
%     \centering
%     \includegraphics[width=0.9\linewidth]{figures/delta_M1.pdf}
%     \caption{Effect of individual biases on each element of $M$. As expected, we see an increasing trend of MSE with respect to $\Delta$.}
%     \label{fig:delta-mMSE}
% \end{figure}

\begin{figure}[t]%
    \centering
    \subfloat[\centering \label{fig:delta-mMSE}]{{\includegraphics[width=0.48\linewidth]{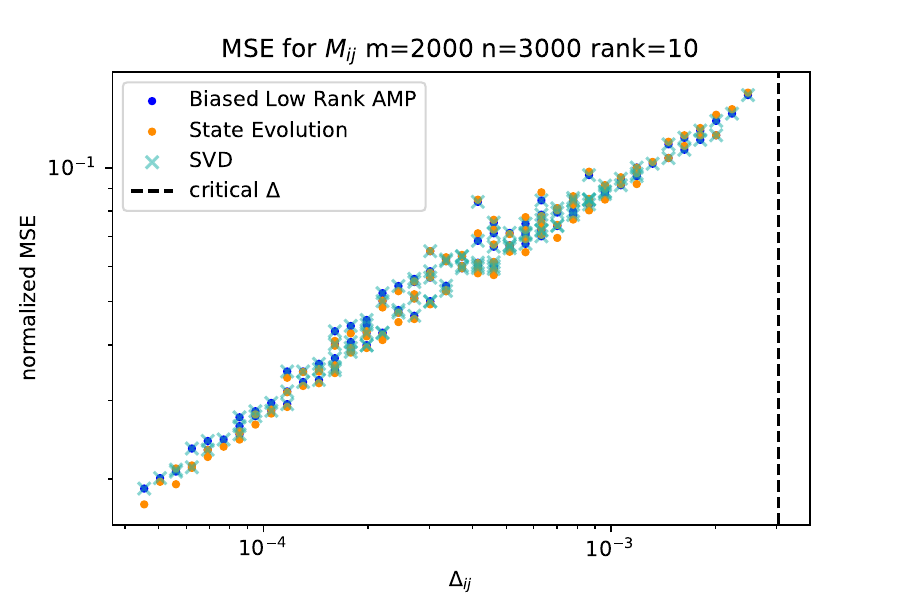} }}%
    \qquad
    \subfloat[\centering \label{fig:singular}]{{\includegraphics[width=0.42\linewidth]{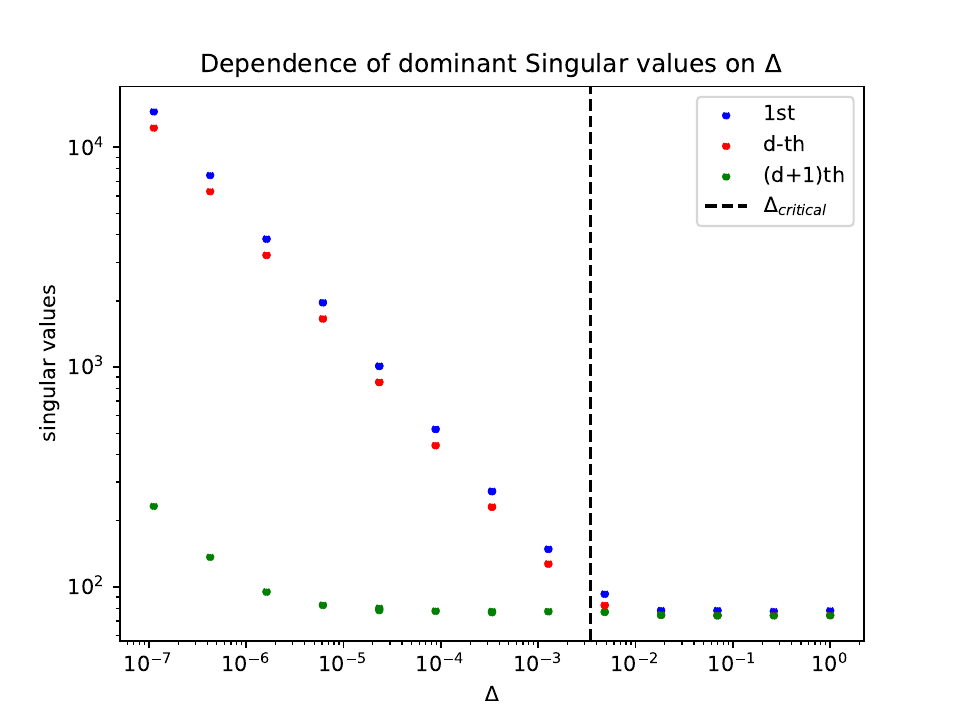} }}%
    \caption{(a) Effect of individual biases on each element of $M$. As expected, we see an increasing trend of MSE with respect to $\Delta$. (b) The dominant singular values of $\tl{Y}$ are affected by $\Delta$. If $\Delta$ exceeds the critical value, the first $d$ singular values will not be distinguishable from the other singular values.}%
    \label{fig:MSE-Fisher}%
\end{figure}

\subsection{Effect of inverse Fisher information on the singular values of the observation matrix}\label{sub:delta-vs-singular}
%An additional experiment shows how $\Delta$ values affect the singular values of $\tl{Y}$. 
We show that staying below the critical inverse Fisher information \eqref{eq:delta-critical} is indeed crucial for estimation. To do so, we conduct an experiment where we set bias terms $s^u_i = u;\quad \forall i \in [m]$ and  $s^v_j = v;\quad \forall j \in [n]$ and then we vary $u$ and $v$ in the range $[0,8]$. Next, we plot the first, the $d$-th, and the $(d+1)$-th singular values of $\tl{Y}$ with respect to $\Delta = e^{-(u+v)}$. It should be noted that in this experiment, for each instance, all $M_{ij}$'s have the same bias. We set $m=1000, n=2000, d=10$. Fig. \ref{fig:singular} shows how these singular values are indistinguishable when $\Delta$ exceeds the critical value.
% \begin{figure}[ht]
%     \centering
%     \includegraphics[width=0.8\linewidth]{figures/singular.pdf}
%     \caption{The dominant singular values of $\tl{Y}$ are affected by $\Delta$. If $\Delta$ exceeds the critical value, the first $d$ singular values will not be distinguishable from the other singular values.}
%     \label{fig:singular}
% \end{figure}

\begin{table*}[t]
\centering
    \caption{Parameter selection for constructing the document-word co-occurance matrix using CountVectorizer function.}  
    \label{tab:nlp}
    \begin{tabular}{ c|c|c|c|c }
            \hline
            \multicolumn{5}{c}{CountVectorizer Parameters} \\
            \hline
             \textbf{mindf} & \textbf{maxdf} & \textbf{stopwords} & \textbf{preprocessor} & \textbf{tokenizer} \\
             \hline
             $10$ & $3000$ & "english" & remove digits and special signs & lemmatization\\ 
             \hline         
    \end{tabular}         
\end{table*}

\subsection{Evaluating the algorithm on a real text dataset}
\label{simulations:real}
Finally, we apply our proposed algorithm over text data from a publicly available dataset called Large Movie Review Dataset (\cite{maas2011sent}). 
This dataset includes texts with positive and negative sentiment. We select a batch of $7000$ reviews at random and apply the following preprocessing:  
We use the ''CountVectorizer" function of the Python Scikit-learn library (\cite{scikit-learn}) to count the number of word occurances in each document. We set the parameters of this function according to Table \ref{tab:nlp}.
These selections give us $m=7000$ and $n=8139$. This co-occurance matrix will be the $Z$ that describes how many times each word occurs in each of the documents. Thus, $X_1$ and $X_2$ will refer to documents and words, respectively.  Since the ``true" embedding vectors are not known,
we first run Algorithm~\ref{alg:low-ramp}, the biased
low-rank AMP algorithm, to find an approximation of the
true embedding vectors.
We assume a rank $d=10$ and use the $L_2$ denoisers with  $\lambda=10^{-3}$ for $10$ iterations and save the final results as the ground truths $U_0$ and $V_0$. The resulting matrices might not be zero-mean, hence we subtract the row mean from each matrix. Furthermore, in order to avoid very small matrix entries, we normalize each matrix by dividing all elements by the smallest element on that matrix.

% use this full dataset to approximate of the true embedding vectors
% on the full dataset, and then see how well these
% are estimated from a sub-sample of the dataset.
% Now, following section \ref{sub:bias-est} we can calculate the bias terms and construct our problem. Please note that since we will need the sum of exponents of bias terms, we make the following rough assumptions:
% %\begin{subequations}
% \begin{align}
%     \sum_{i=1}^m\exp(\wh{s}^u_i) \approx m \exp^{10}\quad
%     \sum_{j=1}^n\exp(\wh{s}^v_j) \approx n \exp^{10}
% \end{align}   
% %\end{subequations}
% \textcolor{red}{We note that these assumptions do not contradict those in \eqref{eq:bias_norm} since we can simply multiply the expression in \eqref{eq:Pij} by a different normalization parameter $C$.}
% Moreover, we need to make an assumption on the rank of the matrix.

Next, we sample $m=2000$ and $n=3000$ rows of $U_0$ and $V_0$, respectively. Using these samples, we construct a new Poisson channel following section \ref{sec:poisson} to derive a new $Z$ matrix that is observed through the channel. Now, we apply algorithms \ref{alg:low-ramp} and \ref{alg:se}.
%and yield the results in section \ref{simulations:real}.
% More specifically, we select a large batch of text documents and perform preprocessing steps explained in Appendix section \ref{app:realdata}. 
% Next, we construct a document-word co-occurance matrix using the data. This will serve as our $Z$ matrix. We estimate bias vectors $s^u$ and $s^v$ using \eqref{eq:sest} and use the estimations to compute $\tl{Y}$. At this stage, we do not have the ground truth distributions. Thus, we apply algorithm \ref{alg:low-ramp} and use the output $\wh{A}$ and $\wh{B}$, and corresponding $\wh{U}$ and $\wh{V}$ as the ground truths $U$ and $V$, respectively.
% Now, we sample $m$ and $n$ rows from ground truth $U$ and $V$'s, respectively and observe the new matrix $Z$ from the Poisson channel. We apply algorithms \ref{alg:low-ramp} and \ref{alg:se} to derive the final estimations $\wh{A}_k$ and $\wh{B}_k$ and corresponding $\wh{U}_k$ and $\wh{V}_k$. 
Fig. \ref{fig:real-loss} and Fig. \ref{fig:real-mse} show the resulting loss and MSE when we sample $m=2000$ and $n=3000$ from the ground truth distributions. Again, we see an excellent match between the SE and the simulations.
% \begin{figure}[ht]
%     \centering
%     \includegraphics[width=0.8\linewidth]{figures/realdata_loss.pdf}
%     \caption{Loss function vs iteration when sampling from a real dataset.}
%     \label{fig:real-loss}
% \end{figure}
% \begin{figure}[ht]
%     \centering
%     \includegraphics[width=0.8\linewidth]{figures/realdata_mse.pdf}
%     \caption{MSE vs iteration when sampling from a real dataset.}
%     \label{fig:real-mse}
% \end{figure}

\begin{figure}[ht]%
    \centering
    \subfloat[\centering \label{fig:real-loss}]{{\includegraphics[width=0.45\linewidth]{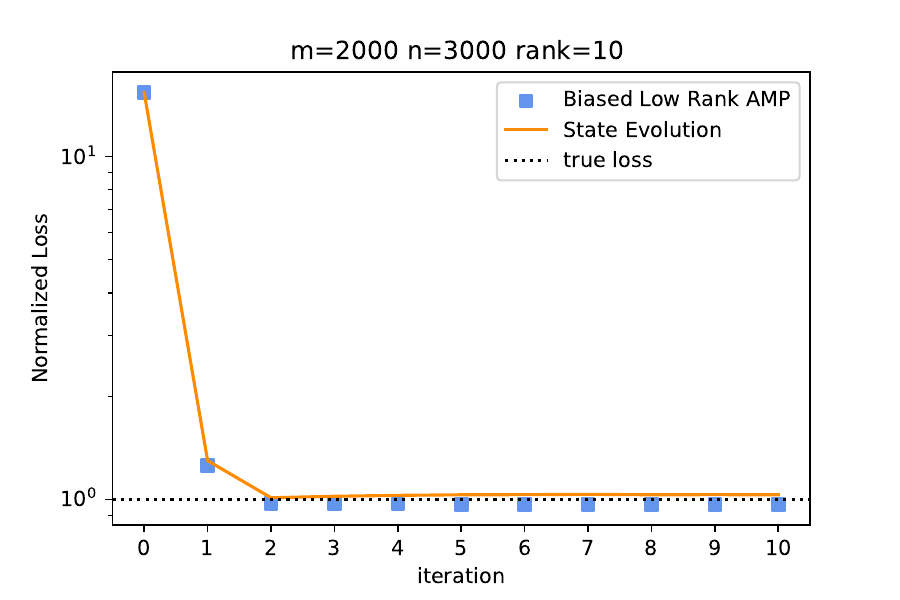} }}%
    \qquad
    \subfloat[\centering \label{fig:real-mse}]{{\includegraphics[width=0.45\linewidth]{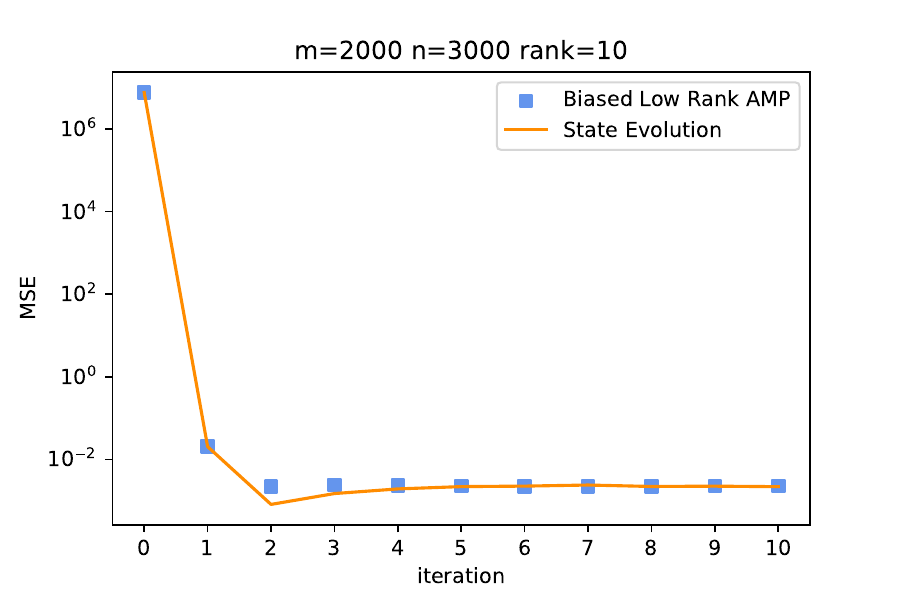} }}%
    \caption{Loss function (a) and MSE (b) vs iteration when sampling from a real dataset.}%
    \label{fig:real}%
\end{figure}

%% file: sections/examples.tex
%\section{Examples}
%\textcolor{red}{Move this section later}
We consider experiments with denoisers for two standard
regularizers:  
squared-norm ($L_2$) and sparsity-inducing ($L_1$).
\subsubsection{Squared-norm regularizers}
% For the squared norm reguarlizer 
% \eqref{eq:l2reg}, it can be verified that
% the denoisers are given by:
% \begin{subequations}
% \begin{align}
%    G_a([P^a_k]_{i\ast},r^u_i,F^a_k) &= [P^a_k]_{i\ast}(F^a_k+\lambda_u r^u_i I_d)^{-1}\\ 
%   G_b([P_k^b]_{j\ast},r^v_j, F_k^b) &= [P_k^b]_{j\ast} (F_k^b + \lambda_v r^v_j I_d)^{-1}   
% \end{align}    
% \end{subequations}
In this case, the regularizers are given by: 
\begin{equation} \label{eq:regl2}
    \phi_u(U) = \frac{\lambda_u}{2} \sum_{i=1}^m \|\bs{u}_i\|_2^2,
    \quad
    \phi_v(V) = \frac{\lambda_v}{2} \sum_{j=1}^n \|\bs{v}_j\|_2^2
\end{equation}
A standard least-squares calculation shows that the 
denoisers \eqref{eq:Gadef}
are given by:
\begin{subequations}
\begin{align}
   G_a([P^a_k]_{i\ast},r^u_i,F^a_k) &= [P^a_k]_{i\ast}(F^a_k+\frac{\lambda_u}{r^u_i}  I_d)^{-1}\\ 
  G_b([P_k^b]_{j\ast},r^v_j, F_k^b) &= [P_k^b]_{j\ast} (F_k^b + \frac{\lambda_v}{r^v_j}  I_d)^{-1}   
\end{align}    
\end{subequations}

\subsubsection{Sparsity inducing regularizers}
In this case, the regularizers are given with the $L_1$-norm:
\begin{equation} \label{eq:regsparse}
    \phi_u(U) = \lambda_u \sum_{i=1}^m \|\bs{u}_i\|_1,
    \quad
    \phi_v(V) = \lambda_v \sum_{j=1}^n \|\bs{v}_j\|_1
\end{equation}
The denoiser \eqref{eq:Gadef}
can then be implemented with a LASSO problem.
Let $a_i = [A]_{i\ast}^\intercal$ (a column vector).  Then, the denoiser optimization \eqref{eq:Gadef-row} can be written as:
\begin{equation}\label{eq:l1-denoiserA}
    G_a([P^a_k]_{i\ast},r_i^u,F^a_k) = \argmin_{a} \frac{1}{2} \|W^a a - q\|^2_2+ \frac{\lambda_u}{\sqrt{r_i^u}}  \|a\|_1
\end{equation}
where
\begin{equation}
    W^a = (F^a_k)^{1/2}, \quad q = (W^a)^{-1}[P_k^a]_{j\ast}^\intercal.
\end{equation}
The denoiser $G_b(\cdot)$ is defined similarly.
% Similarly, 
% \begin{equation}\label{eq:l1-denoiserB}
%     G_b([P^b_k]_{j\ast},r_j^v,F^b_k) = \argmin_{b} \frac{1}{2} \|W^b b - q\|^2_2+ \frac{\lambda_v}{\sqrt{r_j^v}}  \|b\|_1
% \end{equation}
% for some $W^b$ and $q$ such that:
% \begin{align}
%     F_k^b &:= (W^b)^\intercal W^b; \quad
%     [P_k^b]_{j\ast}^\intercal := (W^b)^\intercal q 
% \end{align}
% Simple calculus shows that:
% \begin{align}
%         \Gamma^a_k &= (F^b_k)^{-1};\quad
%         \Gamma^b_k = (F^a_k)^{-1}
% \end{align}

%% file: sections/conclusions.tex
\section{Conclusions}
We have proposed a simple Poisson model
to study learning of embeddings.  Applying
an AMP algorithm to this estimation problem
enables predictions of how key parameters
such as the embedding dimension, number
of samples and relative frequency impact
embedding estimation.  Future work
could consider more complex models,
where the embedding correlations 
are described by a neural network.
Also, we have assumed that the
embedding dimension is known.  An interesting
avenue is to study the behavior of the methods
in both over and under-parameterized regimes.